\documentclass[journal]{IEEEtran}

\usepackage{epsfig}
\usepackage{amssymb}
\usepackage{amsthm}
\usepackage{graphicx}
\usepackage{xcolor}
\usepackage{boldline,multirow}
\usepackage{hyperref}
\usepackage{amsmath}
\usepackage{amsthm}
\usepackage{makecell} 

\newtheorem{theorem}{Theorem}

\usepackage{pifont}

\begin{document}

\title{Robust Offset-Free Lane-Keeping Control Framework for Vision-Only Automated Vehicles}

\author{Der-Hau Lee
\thanks{The author was affiliated with the Department of Electrophysics,  National Yang Ming Chiao Tung University, Hsinchu 300, Taiwan. e-mail: derhaulee@gmail.com.}
}

\maketitle

\begin{abstract}
Achieving accurate, real-time lane-keeping control during high-speed cornering on large-curvature roads while ensuring robust vehicle stability is a critical challenge in autonomous driving. To address this challenge, this study developed a robust offset-free model predictive control (OF-MPC) framework for vision-based autonomous vehicle lane-keeping. The regulation problem in this framework is solved using several algorithms. Numerical simulation results indicated that the interpolation-tube-based constrained iterative linear quadratic regulator (itube-CILQR) algorithm outperformed solvers based on the standard  Interior Point Optimizer (IPOPT) and Operator Splitting Quadratic Program (OSQP), achieving a mean computation time of less than 5 ms. The itube-CILQR algorithm was then used for vehicle control in vision-based lane-keeping experiments in a simulation environment to quantify the perceptual uncertainty of a visual perception model (multitask UNet). This quantification was performed using the OF-MPC  framework for traffic scenes that were not encountered during model training. Overall, the developed OF-MPC  framework and the proposed method for characterizing the perceptual uncertainty of a visual perception system in unseen environments can contribute to the future development of advanced autonomous driving systems.
\end{abstract}


\section{Introduction}
\IEEEPARstart{O} {}ffset-free model predictive control (OF-MPC) is an advanced control strategy for solving the steady-state error problem in the presence of internal disturbances (plant\textendash model mismatch) or external disturbances (measurement noise), which cannot be accurately modeled in standard  MPC \cite{Raw03, Aut09, Aut12, Wal15, Raw26}. An OF-MPC system typically consists of three components, namely a state estimator, target selector, and regulator. By augmenting the system state with an integrating disturbance, OF-MPC can eliminate the discrepancies between the steady states of system models and achieve offset-free control performance. OF-MPC was originally developed for industrial process control, and its control mechanisms were then incorporated into economic MPC, which aims to optimize economic objectives, such as maximizing profit, rather than primarily track predefined setpoints \cite{Pan19, Li22}. Although numerous studies have employed OF-MPC  in automated industrial processes, limited studies have used it in automated automotive applications.

In \cite{Ge22}, an OF-MPC framework was developed for linear systems to address common model mismatches and disturbances in the lateral control of autonomous vehicles, including tire stiffness mismatches, steering angle offsets, and lateral slope disturbances, all of which contribute to steady-state tracking errors and degrade vehicle dynamic performance. The framework integrates these mismatches and disturbances into a single disturbance term, which is estimated using the Kalman filter; the term is then incorporated into an optimization process to  eliminate lateral steady-state tracking errors. The study revealed that this framework could efficiently complete  calculations in less than 7 ms. On the basis of this framework, \cite{Ge22c}  developed an OF-MPC framework for the longitudinal and lateral control of autonomous vehicles to achieve offset-free trajectory tracking. In this framework, the friction circle is incorporated into motion constraints to prevent vehicle tires  from losing adhesion. In addition, nonlinear constraints derived from the friction circle are  transformed  into linear constraints through polygon approximation. This framework exhibited an average computation time of approximately 5 ms in simulations and real-world experiments with a vehicle speed of  25.2 km/h. In \cite{Yi24}, an OF-MPC scheme incorporating a moving-horizon estimation (MHE) approach was developed to compensate for the heading bias resulting from localization uncertainty in lateral motion planning for autonomous buses; the measurable variables were yaw rate, heading error, and lateral position error, whereas the disturbance variables were heading-error bias, road-curvature bias, and control-input bias. Thus, the disturbance vector had the same number of dimensions (three)\ as the measurement vector did. This scheme required average computational times of 1.26 and 3.22 ms for solving the MPC and MHE problems, respectively, in lateral motion planning during vehicle tests with a maximum speed of 40 km/h. An OF-MPC framework was proposed in \cite{Che24} to address model mismatches, parameter variations, and unexpected disturbances in integrated brake system control for autonomous vehicles. The proposed framework incorporates the total disturbance as a measurable disturbance input. The performance of the framework was assessed through simulations and experiments, and its suitability for braking systems was verified on a microchip; the framework exhibited an average computation time of 70 $\mu$s. In \cite{Dai25}, an offset-free economic MPC framework was developed to minimize the tracking error and control effort of an autonomous electric vehicle while accounting for economic criteria. The offset-free mechanism of this framework compensates for disturbances and prediction errors resulting from self-triggered control operations. Real-world experiments revealed that the developed framework effectively tracked the reference trajectory of an autonomous electric vehicle operating at a speed of 2 m/s, requiring less than 0.05 s to generate each control update. In summary, the preceding discussion suggests that OF-MPC is a promising approach for controlling vehicles in resource-limited automotive applications owing to its low computational burden and real-time computation performance. 

Lane-keeping control is a type of lateral vehicle control, although the corresponding MPC problem typically involves constrained nonlinear optimization because of road curvature \cite{Xu20}. Solving this problem becomes particularly challenging in high-speed cornering scenarios  on roads with tight curves, when limited on-board computational resources may hinder the generation of real-time problem solutions and timely control actions, thereby compromising driving safety. To mitigate this problem, we previously proposed a tube-based robust control strategy for vision-guided autonomous vehicles \cite{Lee26}. Specifically, a computationally efficient interpolation-tube-based constrained iterative linear quadratic regulator (itube-CILQR) algorithm was developed and incorporated into the vision perception system proposed in \cite{Lee24} for autonomous vehicle lane keeping.  The feasibility of this algorithm was validated through numerical simulations and vision-based experiments conducted  in The Open Racing Car Simulator (TORCS) \cite{Wym13, Che15, Car15, Li19, Liu24t}. Nevertheless, certain limitations were observed in \cite{Lee26}. First, the itube-CILQR algorithm does not incorporate output feedback, which is crucial in  autonomous driving applications. Second, this algorithm does not address lateral steady-state tracking errors, resulting in lateral position deviations from the optimal vehicle trajectory along curvy road sections. Third, the perceptual uncertainty of the adopted visual perception system \cite{Lee24} was not quantified in previously unseen  scenarios. Quantifying perceptual uncertainty in unseen environments is crucial for developing high-performance visual perception and autonomous driving systems.

To address the aforementioned limitations, the present study developed a robust OF-MPC framework, incorporated it into the visual perception system \cite{Lee24}, and then assessed the perceptual uncertainty of this system in unseen environments. By modeling the nonlinear road curvature term as an integrating disturbance, the OF-MPC framework efficiently solves the lane-keeping control problem in high-speed cornering scenarios on curvy roads. It achieves near-zero steady-state tracking errors while maintaining robust, stable, and real-time performance. The main contributions of this study are outlined as follows:
\begin{enumerate}
  \item We developed a robust offset-free lane-keeping control framework for autonomous driving that requires only visual information as input data. The regulation problem in the offset-free control framework is solved using the itube-CILQR control algorithm, which ensures the robustness of real-time lane-keeping control without steady-state errors for autonomous vehicles navigating tight curves.
  \item We validated the performance of the proposed control framework through numerical simulations and vision-based experiments. The vision-based experiments were conducted in TORCS in high-speed cornering scenarios on roads with large curvature. Moreover, we quantified the perceptual uncertainty of the adopted visual perception framework through vision-based experiments in unseen driving scenes.
\end{enumerate}

The remainder of this paper is organized as follows. Section II presents the adopted lane-keeping dynamics model. Section III describes the developed OF-MPC framework for lane keeping. Section IV introduces the itube-CILQR algorithm. Section V presents the experimental results. Finally, Section VI provides the conclusion.

\section{Vehicle Lateral Dynamics Model for the Lane-Keeping Problem}
This study used a linearized bicycle model  to develop a representation of vehicle lane-keeping dynamics \cite{Xu20, Lee19,Lev22}. Compared with nonlinear models, this linear model can effectively reduce the computational complexity of MPC, thereby facilitating the real-time implementation of this approach in time-critical scenarios, such as high-speed cornering on tightly curvy roads. The discrete-time bicycle model used in the proposed OF-MPC framework can be described by the following state-space representation:
\begin{subequations}
\begin{equation}
{\bf x}_{i + 1}  \equiv {\bf f}\left( {{\bf x}_i ,{\bf {u}}_i ,{\bf {w}}_i} \right)  = {\bf A{x}}_i  + {\bf B{u}}_i+ {\bf w},
\end{equation}
\begin{equation}
{\bf y}_{i + 1}    = {\bf C{x}}_i, 
\end{equation}
\end{subequations}
where  $\bf A$, $\bf B$, and $\bf w$ are state-space matrices. These matrices are expressed as follows: 
\[
\begin{array}{l}
 {\bf A} = \left[ {\begin{array}{*{20}c}
   {\alpha _{00} } & {\alpha _{01} } & 0 & 0  \\
   0 & {\alpha _{11} } & {\alpha _{12} } & {\alpha _{13} }  \\
   0 & 0 & {\alpha _{22} } & {\alpha _{23} }  \\
   0 & {\alpha _{31} } & {\alpha _{32} } & {\alpha _{33} }  \\
\end{array}} \right], \\ 
 {\bf B} = \left[ {\begin{array}{*{20}c}
   0  \\
   {\beta _0 }  \\
   0  \\
   {\beta _1 }  \\
\end{array}} \right],\quad {\bf w} = \kappa\left[ {\begin{array}{*{20}c}
   0  \\
   { \gamma _0 }  \\
   0  \\
   { \gamma _1 }  \\
\end{array}} \right] . \\ 
 \end{array}
\]
The corresponding matrix coefficients are given as follows:
\[
\begin{array}{l}
\alpha _{00} = \alpha _{22} = 1, \quad \alpha _{01} = \alpha _{23} = dt, \\
\alpha _{11}  = 1 - \frac{{2\left( {C_{\alpha f}  + C_{\alpha r} } \right)dt}}{{mv_x }},\quad \alpha _{12}  = \frac{{2\left( {C_{\alpha f}  + C_{\alpha r} } \right)dt}}{m}, \\
\alpha _{13}  =  - \frac{{2\left( {l_f C_{\alpha f}  - l_r C_{\alpha r} } \right)dt}}{{mv_x }},\quad \alpha _{31}  =  - \frac{{2\left( {l_f C_{\alpha f}  - l_r C_{\alpha r} } \right)dt}}{{I_z v_x }} , \\
\alpha _{32}  = \frac{{2\left( {l_f C_{\alpha f}  - l_r C_{\alpha r} } \right)dt}}{{I_z }},\quad \alpha _{33}  = 1 - \frac{{2\left( {l_f^2 C_{\alpha f}  + l_r^2 C_{\alpha r} } \right)dt}}{{I_z v_x }}, \\
\beta _0 = {\frac{{2C_{\alpha f} dt}}{m}} ,\quad \beta _1  = \frac{{2l_f C_{\alpha f} dt}}{{I_z }}, \\
\gamma _0  =  - \frac{{2\left( {l_f C_{\alpha f}  - l_r C_{\alpha r} } \right)dt}}{m} - v_x^2 dt, \quad
\gamma _1  =  - \frac{{2\left( {l_f^2 C_{\alpha f}  + l_r^2 C_{\alpha r} } \right)dt}}{{I_z }}.\\
\end{array}
\]
Here, the control input ${\bf {u}} \equiv  \left[  {\delta}  \right]$, where $\delta$ is the front-wheel steering angle. The state vector  ${\bf x}  = \left[ {\begin{array}{*{20}c}
   {x_0 } & {x_1 } & {x_2 } & {x_3 }  \\
\end{array}} \right]^T  \equiv  \left[ {\begin{array}{*{20}c}
   \Delta  & {\dot \Delta } & \theta  & {\dot \theta }  \\
\end{array}} \right]^T$ consists of the ego vehicle's lateral offset with respect to the lane centerline (${\Delta}$),  its heading angle (${\theta}$), and the derivatives of these two terms (${\dot \Delta }$ and ${\dot \theta }$). The matrix ${\bf w}$, which includes the road curvature $\kappa$, represents an external  disturbance. The overall system expressed in (1a) and (1b) is subject to the stage constraints ${\bf x}_{i \ne N} \in \mathbb{X}$ and ${\bf u}_{i \ne N} \in \mathbb{U}$, the terminal constraint ${\bf x}_{N} \in \mathbb{T}$, and the disturbance ${\bf w} \in \mathbb{W}$. Moreover, the model parameters $l_f$ and $l_r$ are the distances from the center of gravity of the ego vehicle to the front and rear axles, respectively; $C_{\alpha f}$ and $C_{\alpha r}$ are the tire cornering stiffness values of the front and rear tires, respectively;  $m$ is the vehicle mass; $I_z$ is the vehicle moment of inertia about the vertical direction at the center of mass; and $dt$ is the sampling time. The vehicle and model parameters selected in this study are listed in Table I. A disturbance-free nominal system of (1a) is defined as follows:
\begin{equation}
{\bf \bar x}_{i + 1}  = {\bf A \bar x}_i  + {\bf B \bar u}_i ,
\end{equation}
where $\bf \bar  x$ and $\bf \bar  u$  are the associated nominal state and control vectors, respectively.

\begin{table}[!t]
\caption{Model and Vehicle Parameters Used in This Study}
\begin{center}
\begin{tabular}{llll}
\Xhline{1.5pt} 
Symbol &  Values & Symbol &  Values  \\
\hline
$dt$ &  0.01 s & ${C_{\alpha f} }$  & 80000 N/rad  \\
$v_x$ & \{20.0, 21.1\} m/s  &  ${C_{\alpha r} }$ &  80000 N/rad  \\
$I_z$ &   2000 kgm$^{2}$ &  $l_f$  & 1.27 m \\
$m$ &  1150 kg & $l_r$ & 1.37 m  \\
\Xhline{1.5pt} 
\end{tabular}
\end{center}
\end{table}

The output signal {\bf y}  is a vector comprising the measurable state variables of the controlled system. Additionally, the output matrix {\bf C} maps system states to outputs on the basis of the types of sensors in the ego vehicle. Specifically, the state variable $\dot \theta$ can be obtained from inertial measurement unit measurements \cite{Lev22}, whereas the state variable $\dot \Delta$ (or the sideslip angle) can be determined by combining inertial navigation system  and GPS measurements \cite{JCG06}. The state variables $\Delta$ and $\theta$ can be estimated using a trained deep neural network model that extracts features from images captured by a front-facing monocular camera \cite{Lee24, Lee25}. In general, an increase in the number of measurable state variables improves control performance. Herein, we consider a challenging scenario of vision-only autonomous driving, where the ego vehicle is equipped with only a single camera and only the states $\Delta$ and $\theta$ can be observed. Therefore, the output vector is expressed as ${\bf y} \equiv \left[ {\begin{array}{*{20}c}
   \Delta  & \theta   \\
\end{array}} \right]^T 
$, and the output matrix ${\bf C}$ is defined as follows:
\[
{\bf C} \equiv \left[ {\begin{array}{*{20}c}
   1 & 0 & 0 & 0  \\
   0 & 0 & 1 & 0  \\
\end{array}} \right].
\]

In summary, the state, input, and output vectors required to describe lane-keeping dynamics are defined as 
${\bf x} \in \mathbb{R}^{n=4}  $, ${\bf u} \in \mathbb{R}^{m=1} $, and ${\bf y} \in \mathbb{R}^{p=2} $, respectively. The dimensions of ${\bf u}$ and ${\bf y}$ ($m$ and $p$, respectively) have major effects on the performance of the proposed OF-MPC framework, as demonstrated in the following section.

\section{Proposed OF-MPC Framework for Lane Keeping}
This section describes the proposed OF-MPC framework for lane keeping. This framework is based on previous research \cite{Raw03, Lee26, Raw17} and consists of three primary components: a state estimator, target selector, and  regulator. Section III-A describes each of these components, and Section III-B\ explains the vision-based robust control strategy of the proposed framework, which employs an output feedback strategy. In this paper, the operators $\oplus$ and $\ominus$ denote the Minkowski sum of and Pontryagin difference between sets, respectively.

\subsection{OF-MPC Framework}
\textbf{Initialization}. First, an integrating disturbance vector ${\bf  d} = \left[ {\begin{array}{*{20}c}
    {d_0 }  & {d_1 } \\
\end{array}} \right]^T$ (${\bf d}\in  \mathbb{W}  \subseteq    \mathbb{R}^{n_{d}=2}  $) is incorporated into the nominal state-space model expressed in (2) to compensate for the plant\textendash model mismatch resulting from the curvature term ${\bf w}$ in (1). The resulting model is expressed as follows:
\begin{subequations}
\begin{equation}
{ \bf \bar x}_{i + 1} \equiv {\bf \bar f}\left( {{\bf \bar x}_i ,{\bf \bar u}_i,{\bf  d}_i } \right)  = {\bf A \bar x}_i  + {\bf B \bar u}_i + {{\bf B}_d {\bf  d}}_i,
\end{equation}
\begin{equation}
{\bf  d}_{i + 1}    = {\bf  d}_i, 
\end{equation}
\begin{equation}
{\bf \bar y}_{i + 1}    = {\bf C \bar x}_i +{{\bf C}_d {\bf  d}}_i, 
\end{equation}
\end{subequations}
where
\[
{\bf B}_d  = \left[ {\begin{array}{*{20}c}
   0 & 0 & 0 & 1  \\
   0 & 1 & 0 & 0  \\
\end{array}} \right]^T , \quad {\bf C}_d  = {\bf 0}.
\]
Here, the augmented system satisfies the detectability condition  for the selected matrices ${\bf B}_d$ and ${\bf C}_d$. Because $p>m$ in the considered case, the system cannot be controlled using  OF-MPC  for all the measured variables $\Delta$ and $\theta$  \cite{Raw03, Raw17}. To address this limitation, a controlled-variable vector ${\bf r}\in \mathbb{R}^{n_{c}} $ is defined by selecting a subset of the measured outputs ${\bf y}$ through a linear transformation ${\bf r} = {\bf Hy}$, where ${\bf H}$ is an appropriately selected matrix. This vector then enables offset-free tracking of the selected controlled variables (elements of $\bf r$) under appropriate conditions when $n_c \le m$. In this paper, the most crucial variable in the lane-keeping problem, namely the lateral position deviation ($\Delta$), is selected as the controlled variable. Thus, the controlled-variable vector is defined as  ${{\bf r}_{sp}} ={{{\bf Hy}_{sp}}}=[\Delta _{sp}]\in \mathbb{R}^{1}$, where $\Delta _{sp} = 0$ and  ${\bf H} = \left[ {\begin{array}{*{20}c}
   1 & 0  \\
\end{array}} \right]$.

\begin{table}[!t]
\caption{Ranges of the Control Parameters Considered in This Study}
\begin{center}
\begin{tabular}{llll}
\Xhline{1.5pt} 
Symbol &  Values & Symbol &  Values  \\
\hline
$N$ & 30 & ${\bf Q}$ & diag([20, 1, 20, 1])   \\
$\Delta$  & [-2.0, 2.0] m & ${\bf R}$ & 60$\bf I$  \\
 $\dot \Delta $ &  [13.0, 13.0] m/s & ${\bf Q}'$ & diag([1, 1]) \\
$\theta  $  & [-$\pi  /2$, $\pi  /2$]  rad &  ${\bf R}'$ & 60$\bf I$\\
$\dot \theta  $ & [-10.0, 10.0] rad/s & ${\bf S}$ & 50$\bf I$ \\
$\delta$ &  [-$\pi  /6$, $\pi  /6$]  rad&  $q_s$ & 100 \\
$\kappa$ &  [-0.1, 0.1] 1/m& $q_u$ & 1 \\
${\bf Q}_s$ &  diag([50, 50])& $q_{l1}$ & 50 \\
${\bf R}_s$ &  50$\bf I$&$q_{l2}$ & 20 \\
\Xhline{1.5pt} 
\end{tabular}
\end{center}
\end{table}

\textbf{State Estimator}. The proposed OF-MPC framework begins its iterative computation process by using the Kalman filter to estimate both  current states and disturbances on the basis of measurement outputs. To facilitate this estimation, disturbance terms are incorporated into the augmented system model (3) as follows:
\begin{subequations}
\begin{equation}
{\bf \hat x}_{i + 1}^{a}  = {\bf A}^{a} {\bf \hat x}_i^{a}  + {\bf B}^{a} {\bf u}_i  + {\bf v}_i^x  ,
\end{equation}
\begin{equation}
{\bf \hat y}_i^{}  = {\bf C}^{a} {\bf \hat x}_i^{a}  + {\bf v}_i^y  . 
\end{equation}
\end{subequations}
The augmented state-space matrices are given as follows:
\[
\begin{array}{l}
 {\bf A}^{a}  = \left[ {\begin{array}{*{20}c}
   {\bf A} & {{\bf B}_{\bf d} }  \\
   {\bf 0} & {\bf I}  \\
\end{array}} \right],\quad {\bf B}^{a}  = \left[ {\begin{array}{*{20}c}
   {\bf B}  \\
   {\bf 0}  \\
\end{array}} \right], \\ 
 {\bf C}^{a}  = \left[ {\begin{array}{*{20}c}
   {\bf C} & {{\bf C}_{\bf d} }  \\
\end{array}} \right]. \\ 
 \end{array}
\]
Here ${\bf v}^x$ and ${\bf v}^y$ represent zero-mean white noise sequences with the covariance matrices ${\bf Q}^x$ = diag([$10^{-5}$, $10^{-5}$, $10^{-5}$, $10^{-5}$, 1.0, 1.0])  and ${\bf Q}^y$ = diag([$10^{-5}$, $10^{-5}$]), respectively. The augmented estimated state ${\bf \hat x}^{a}  = \left[ {\begin{array}{*{20}c}
   {{\bf \hat x}}^T & {{\bf \hat d}}^T  \\
\end{array}} \right]^T$ can be updated  using the following expression:
\begin{equation}
\begin{split}
{\bf \hat x}^{a,+}   &= {\bf \hat x}^{a} + {\bf L}\left( {{\bf y} - {\bf \hat y}} \right) \\ 
&= {\bf \hat x}^{a} + {\bf L}\left( {{\bf C}^{a} {\bf x}^{a} - {\bf C}^{a} {\bf \hat x}^{a}} \right).
\end{split}
\end{equation}
The steady-state Kalman gain matrix ${\bf L}$ is expressed as follows:
\begin{equation}
{\bf L} = {\bf P}^{a}{\bf C}^{a,T} \left( {{\bf C}^{a} {\bf P}^{a}{\bf C}^{a,T}  + {\bf Q}^y } \right),
\end{equation}
where ${\bf P}^{a}$ is the covariance matrix of the state predictions, which is initialized as ${\bf P}^{a,0}$ = $10^{-5} \times {\bf Q}^x$. The updated covariance matrix for the next iteration is given as follows:
\begin{equation}
{\bf P}^ {a,next}   = {\bf A}^{a} \left( {{\bf I} - {\bf LC}^{a} } \right){\bf P}^{a}{\bf A}^{a,T} {\rm  + }{\bf Q}^x .
\end{equation}
The resulting updated estimated state (${\bf \hat x}^ +  $) and disturbance (${\bf \hat d}^ +  $) are used in the following steps.

\textbf{Target Selector}. Given the estimated disturbance ${\bf \hat d}^ +  \in  \mathbb{W}$, the target state and control vectors (${\bf x}_s$ and ${\bf u}_s$, respectively) can be obtained by solving the following  optimization problem:

Problem P1 = P1(${\bf \hat d}^ +$, ${\bf r}_{sp}$, $\mathbb{X}$/$\mathbb{U}$):
\begin{subequations}
\begin{equation}
\mathop {\min }\limits_{{\bf x}_s,{\bf u}_s}  J_{xs}  + J_{us} ,
\end{equation}
\begin{equation}
J_{xs}  = \left( {{\bf Cx}_s  + {\bf C}_d {\bf \hat d}^ +   - {\bf y}_{sp} } \right)^T {\bf Q}_s \left( {{\bf Cx}_s  + {\bf C}_d {\bf \hat d}^ +   - {\bf y}_{sp} } \right),
\end{equation}
\begin{equation}
J_{us}  = {{\bf u}_s^T {\bf R}_s{\bf u}_s } 
\end{equation}
\end{subequations}
subject to
\begin{subequations}
\begin{equation}
\left[ {\begin{array}{*{20}c}
   {{\bf I} - {\bf A}} & { - {\bf B}}  \\
   {{\bf HC}} & {\bf 0}  \\
\end{array}} \right]\left[ {\begin{array}{*{20}c}
   {{\bf x}_s }  \\
   {{\bf u}_s }  \\
\end{array}} \right] = \left[ {\begin{array}{*{20}c}
   {{\bf B}_d {\bf \hat d}^ +  }  \\
   {{\bf r}_{sp}  - {\bf HC}_d {\bf \hat d}^ +  }  \\
\end{array}} \right],
\end{equation}
\begin{equation}
{\bf u}_s  \in \mathbb{U},
\end{equation}
\begin{equation}
{\bf Cx}_s  + {\bf C}_d {\bf \hat d}^ +   \in \mathbb{Y}. 
\end{equation}
\end{subequations}
Here, $\mathbb{Y}$ corresponds to the first and the third components of $\mathbb{X}$, and the penalty matrices ${\bf Q}_s$ and ${\bf R}_s$ are positive definite. Problem P1 has low computational complexity and can be solved quickly by using standard optimization software.

\textbf{Regulator}. The regulator (or controller) in the proposed OF-MPC framework solves an MPC problem to guide the system toward the target. Given the target state vector (${\bf x}_s$) and control input vector (${\bf u}_s$) obtained using the steady-state target selector as well as the state and disturbance vectors (${\bf \hat x}^ +  $ and ${\bf \hat d}^ +  $, respectively) estimated using the state estimator, a quadratic programming problem is solved  to minimize deviations  of the  nominal state  sequence (${\bf \bar X} \equiv \left\{ {{\bf \bar x}_0 ,{\bf \bar x}_1 ,...,{\bf \bar x}_N } \right\}$) and nominal control sequence (${\bf \bar U} \equiv \left\{ {{\bf \bar u}_0 ,{\bf \bar u}_1 ,...,{\bf \bar u}_{N - 1} } \right\}$) from their  steady-state targets (${\bf x}_s$ and ${\bf u}_s$, respectively) over a prediction horizon with a length of $N$. This optimization problem is formulated as follows:

Problem P2 = P2\textbf{(${\bf \hat x}^ +$, ${\bf \hat d}^ +$, ${\bf x}_s$/${\bf u}_s$, $\mathbb{X}$/$\mathbb{T}$/$\mathbb{U}$)}:
\begin{subequations}
\begin{equation}
\mathop {\min }\limits_{{\bf \bar X},{\bf \bar U}} J =  J_x +J_u+ J_f  ,
\end{equation}
\begin{equation}
J_x  = \sum\limits_{i = 0}^{N - 1} { \left( {{\bf \bar x}_i  - {\bf x}_s } \right)^T {\bf Q}\left( {{\bf \bar x}_i  - {\bf x}_s } \right) } ,
\end{equation}
\begin{equation}
J_u  = \sum\limits_{i = 0}^{N - 1} { \left( {{\bf \bar u}_i  - {\bf u}_s } \right)^T {\bf R}\left( {{\bf \bar u}_i  - {\bf u}_s } \right) } ,
\end{equation}
\begin{equation}
J_f  = \left( {{\bf \bar x}_N  - {\bf x}_s } \right)^T {\bf P }\left( {{\bf \bar x}_N  - {\bf x}_s } \right)
\end{equation}
\end{subequations}
subject to
\begin{subequations}
\begin{equation}
{\bf \bar x}_{i + 1}  = {\bf \bar f}\left( {{\bf \bar x}_i ,{\bf \bar u}_i,{\bf \hat d}^ + } \right)  ,\quad 0 \le i < N,
\end{equation}
\begin{equation}
{\bf \bar x}_{i\ne N} \in \mathbb{X},
\end{equation}
\begin{equation}
{\bf \bar x}_{N} \in \mathbb{T}, 
\end{equation}
\begin{equation}
{\bf \bar u}_{i\ne N} \in \mathbb{U}, 
\end{equation}
\begin{equation}
{\bf \bar x}_0 = {\bf \hat x}^ + . 
\end{equation}
\end{subequations}
Here ${\bf Q}$ and ${\bf R}$ are symmetric positive-definite weighting matrices with stage costs of $J_u$ and $J_x$, respectively. The penalty matrix ${\bf P}$, which has a terminal cost of $J_f$, is obtained by solving the discrete-time algebraic Riccati equation  (DARE) \cite{Raw26a}, which is expressed as follows:
\begin{equation}
{\bf P} = {\bf A}^T {\bf PA} + {\bf Q} - {\bf A}^T {\bf PB}({\bf B}^T {\bf PB} + {\bf R})^{ - 1} {\bf B}^T {\bf PA}.
\end{equation}
The related control parameters  and constraints used in this study are listed in Table II.

The control signal for the  actual system is then set to ${\bf u} = {\bf \bar u}^*  + {\bf K}({\bf \hat x} - {\bf \bar x})$ , where ${\bf \bar u}^*$ is the first component of the optimal control sequence of Problem P2. An extremely small control gain of $\left| {\bf K} \right| \approx 0$ is selected here for simplicity. The next augmented estimated state is then computed as follows:
\begin{equation}
{\bf \hat x}^{a,next}  = {\bf A}^{a} {\bf \hat x}^{a,+}  + {\bf B}^{a} {\bf \bar u}^* ,
\end{equation}
This state and the measured outputs are substituted into (5) in the next iteration. The updated actual state variables can be obtained by substituting ${\bf u}$ into (1a). Fig. 1 illustrates the architecture of the proposed OF-MPC framework for closed-loop testing  of the controlled system.

\begin{figure}[!t]
\centerline{\includegraphics[scale=0.43]{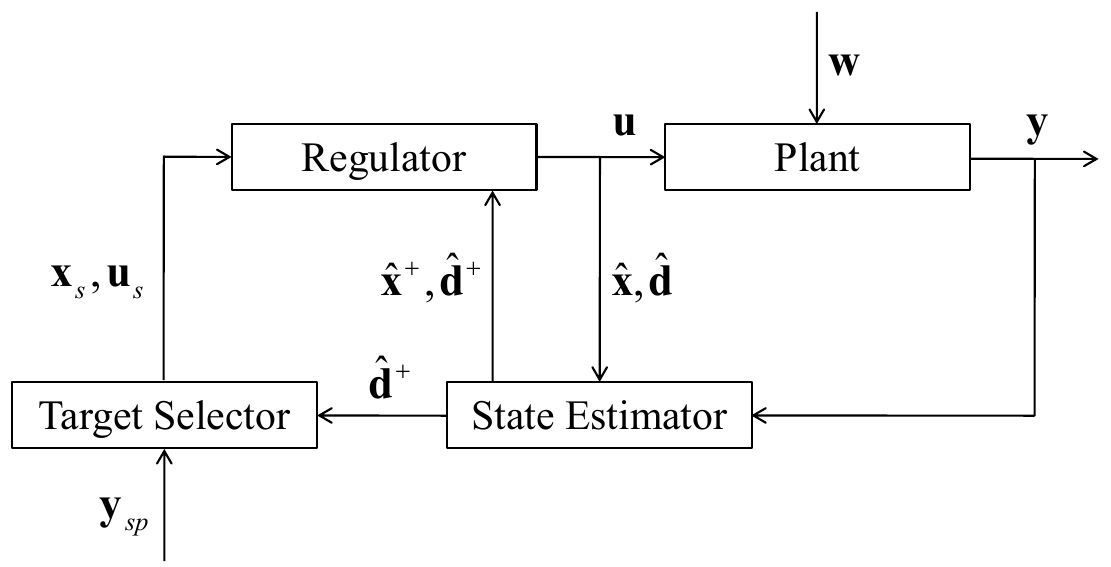}}
\caption{Architecture of the proposed OF-MPC framework for closed-loop testing of the controlled system.}
\end{figure}

\subsection{Vision-Guided Robust Output Feedback Strategy}

In the proposed framework, a tube-based output feedback MPC method \cite{Raw06, Raw17} is employed to ensure system  robustness in the presence of the nonzero road curvature term $\bf w$. Because the set operations required for  the augmented system (3) may have high computational complexity, we use a simplified model (1) for tube-based MPC, as in our previous study \cite{Lee26}. A\ notable improvement in this study over \cite{Lee26} is the use of  an output feedback strategy. Thus, the tightened constraints for the original four-dimensional system ${\bf x}$ are approximated using the two-dimensional subsystem ${\bf x}' \equiv  \left[ {\begin{array}{*{20}c} {\dot \Delta }  & {\dot \theta }  \\
\end{array}} \right]^T \in \mathbb{X}'/\mathbb{T}'$, with the control vector being ${\bf u}' \equiv  \left[  \delta  \right] \in \mathbb{U}'$. Here, $\mathbb{X}'$, $\mathbb{T}'$, and $\mathbb{U}'$ are the constraint sets corresponding to the original sets $\mathbb{X}$, $\mathbb{T}$, and $\mathbb{U}$, respectively. The relevant discrete-time dynamics model is represented as follows \cite{Lev22}:
\begin{subequations}
\begin{equation}
{\bf x'}_{i + 1}  = {\bf A'x'}_i  + {\bf B'u'}_i  + {\bf v}_i^{x'}.
\end{equation}
\begin{equation}
{\bf y'}_i  = {\bf C'} {\bf x'}_i  + {\bf v}_i^{y'}  , 
\end{equation}
\end{subequations}
The corresponding state-space matrices are expressed as follows:
\[
\begin{array}{l}
 {\bf A'} = \left[ {\begin{array}{*{20}c}
   {\alpha _{11} } & {\alpha _{13}  - v_x dt}  \\
   {\alpha _{31} } & {\alpha _{33} }  \\
\end{array}} \right], \\ 
 {\bf B'} = \left[ {\begin{array}{*{20}c}
   {\beta  _0 }  \\
   {\beta  _1 }  \\
\end{array}} \right], \quad {\bf C'} = \left[ {\begin{array}{*{20}c}
   1 & 0  \\
   0 & 1  \\
\end{array}} \right]. \\ 
 \end{array}
\]
Here ${\bf v}^{x'}  \in \mathbb{V}^{x'}$ represents the state disturbance associated with the road curvature bounds defined by $\mathbb{W}$ and ${\bf v}^{y'}  \in \mathbb{V}^{y'}$ represents the predefined output
disturbance. These disturbances are expressed as follows:
\[
\begin{array}{l}
  \mathbb{V}^{x'} \equiv \left[ {\begin{array}{*{20}c}
   { {\kappa _{\min } }\gamma _0 } & {{\kappa _{\max } }\gamma _0 }  \\
\end{array}} \right] \times \left[ {\begin{array}{*{20}c}
   { {\kappa _{\min } }\gamma _1 } & {{\kappa _{\max } }\gamma _1 }  \\
\end{array}} \right]   , \\ 
\mathbb{V}^{y'} \equiv \left[ {\begin{array}{*{20}c}
   { - 0.05} & {  0.05}  \\
\end{array}} \right] \times \left[ {\begin{array}{*{20}c}
   { - 0.05} & {  0.05}  \\
\end{array}} \right]   , \\ 
 \end{array}
\]
where ${\kappa _{\max } }$/${\kappa _{\min } }$ = 0.1/\Pisymbol{@symbol}{45}0.1 (see Table II).

The control gain ${\bf K'}$ for the aforementioned subsystem is set to the LQR gain as follows:
\begin{equation}
{\bf K'} =  - \left( {{\bf B'}^T {\bf P'B'} + {\bf R'}} \right)^{ - 1} {\bf B'}^T {\bf P'A'}.
\end{equation}
Here matrix ${\bf P'}$ is the solution of the DARE. The matrices ${\bf Q'}$ and ${\bf R'}$ have the same parameters as in Problem P2. The observer gain is set to
${\bf L'} = \left[ {\begin{array}{*{20}c}
   {0.057} & {0.057}  \\
\end{array}} \right]$ so that all eigenvalues of $\bf A' + \bf L'C'$ are within the unit circle.

\begin{figure*}[!t]
{\includegraphics*[width=0.25\linewidth]{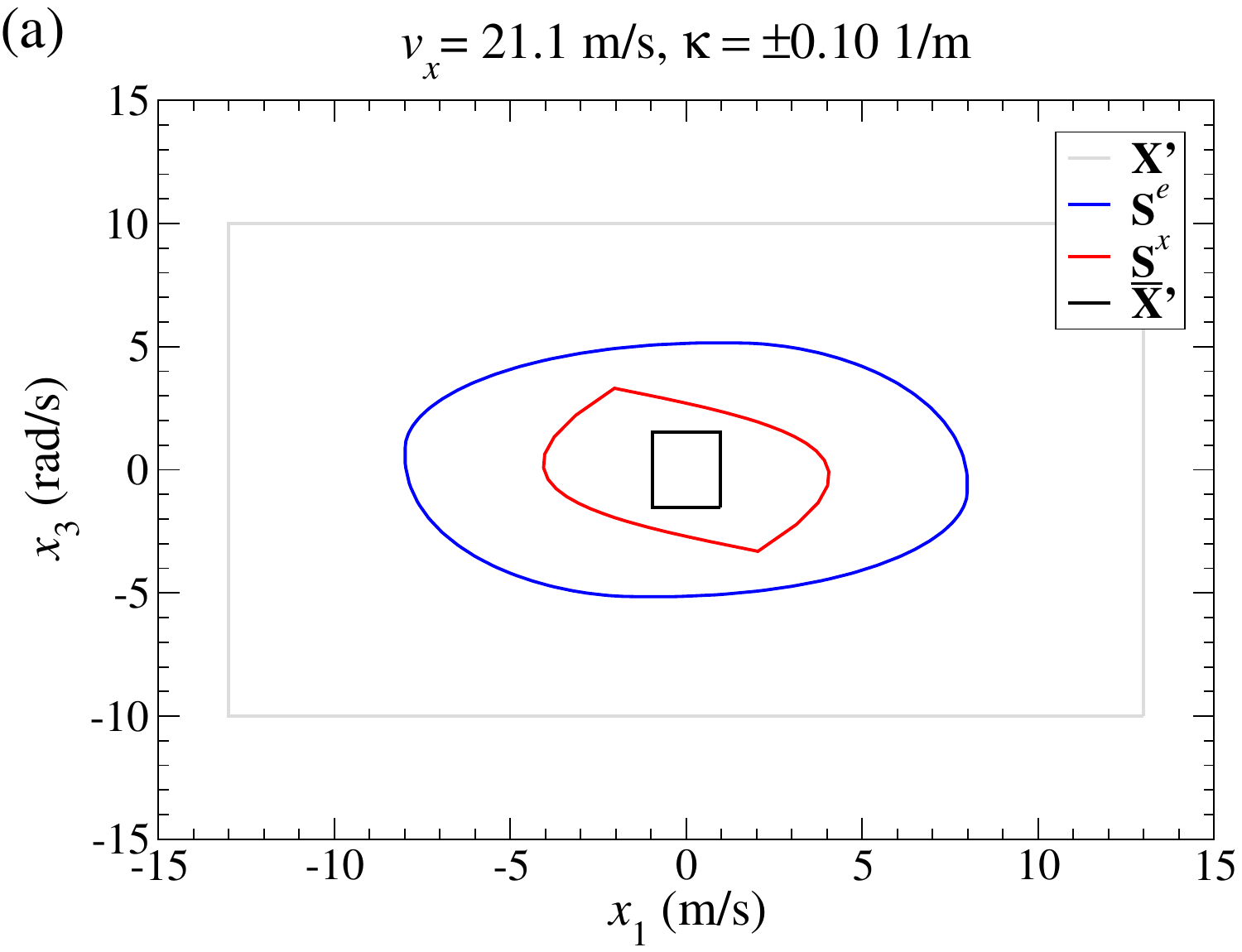}}%
{\includegraphics*[width=0.25\linewidth]{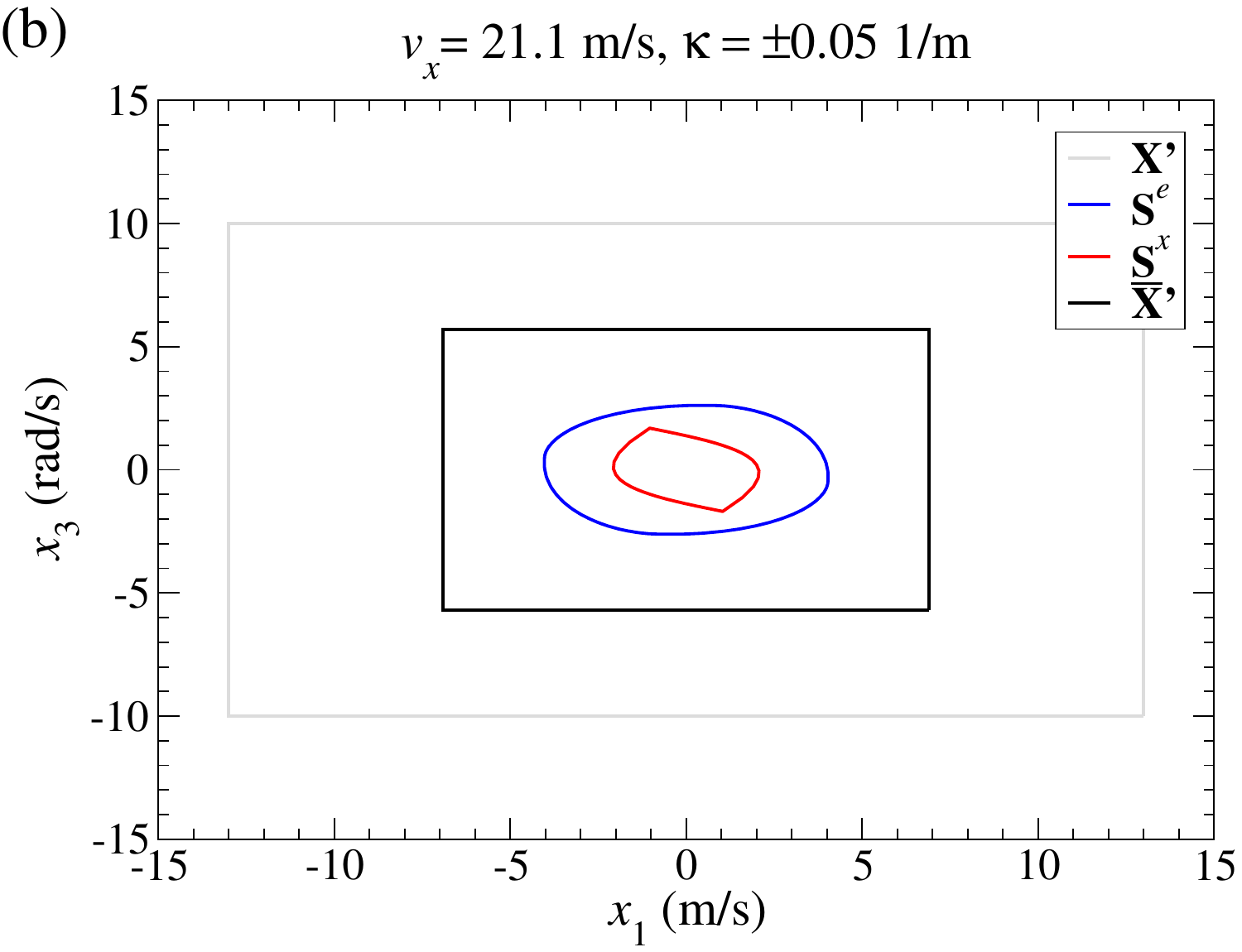}}%
{\includegraphics*[width=0.25\linewidth]{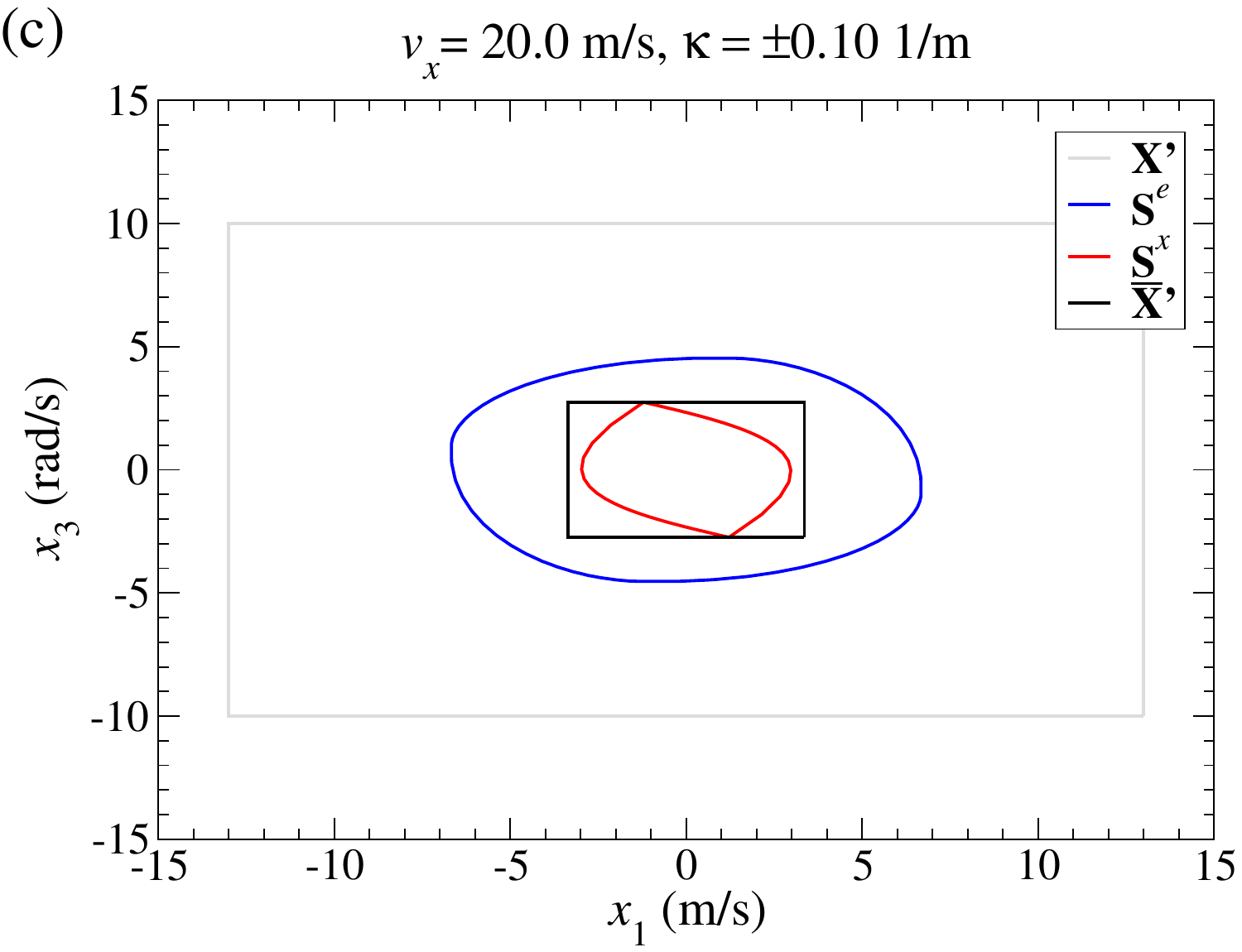}}%
{\includegraphics*[width=0.25\linewidth]{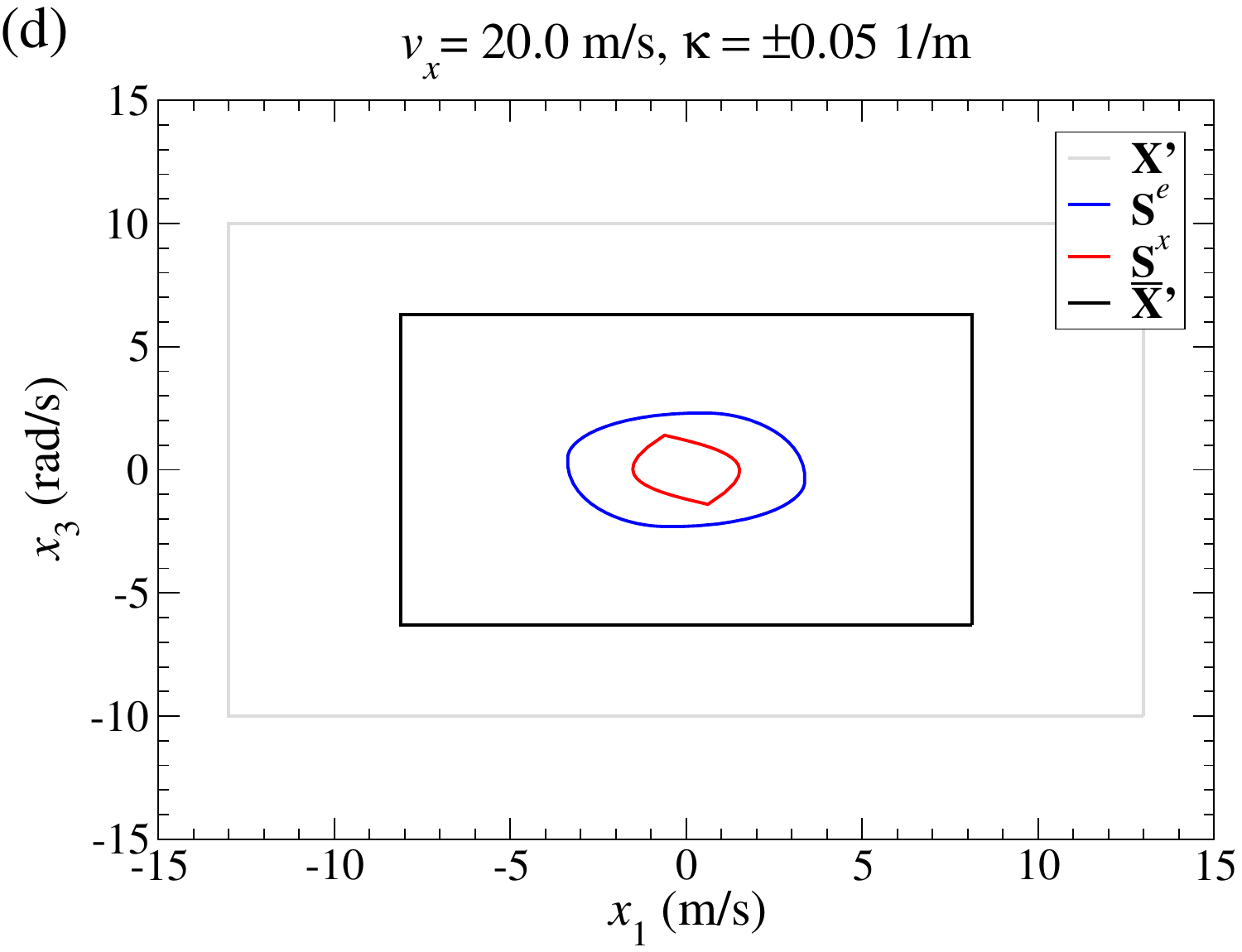}}%
\caption{Outer boundaries of sets $\mathbb{X}'$, $\mathbb{\bar X}'$, $\mathbb{S}^e$,  and $\mathbb{S}^x$ when $v_x$ = 21.1 or 20.0 m/s and $\kappa$ = $\pm$0.1 or $\pm$0.05 1/m. Here, $x_1  \equiv \dot \Delta $ and $x_3  \equiv \dot \theta $.}
\end{figure*}

\begin{figure*}[!t]
{\includegraphics*[width=0.25\linewidth]{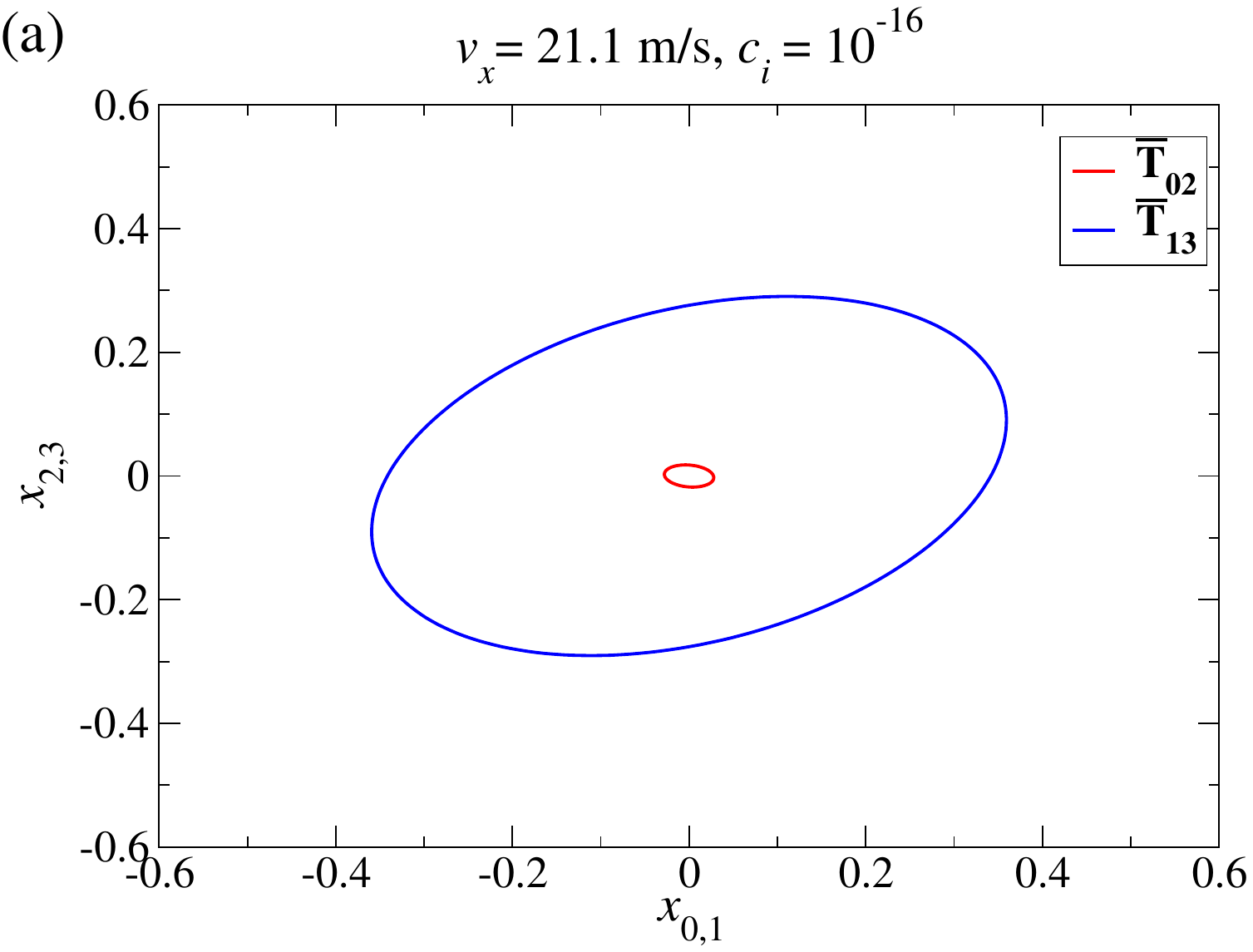}}%
{\includegraphics*[width=0.25\linewidth]{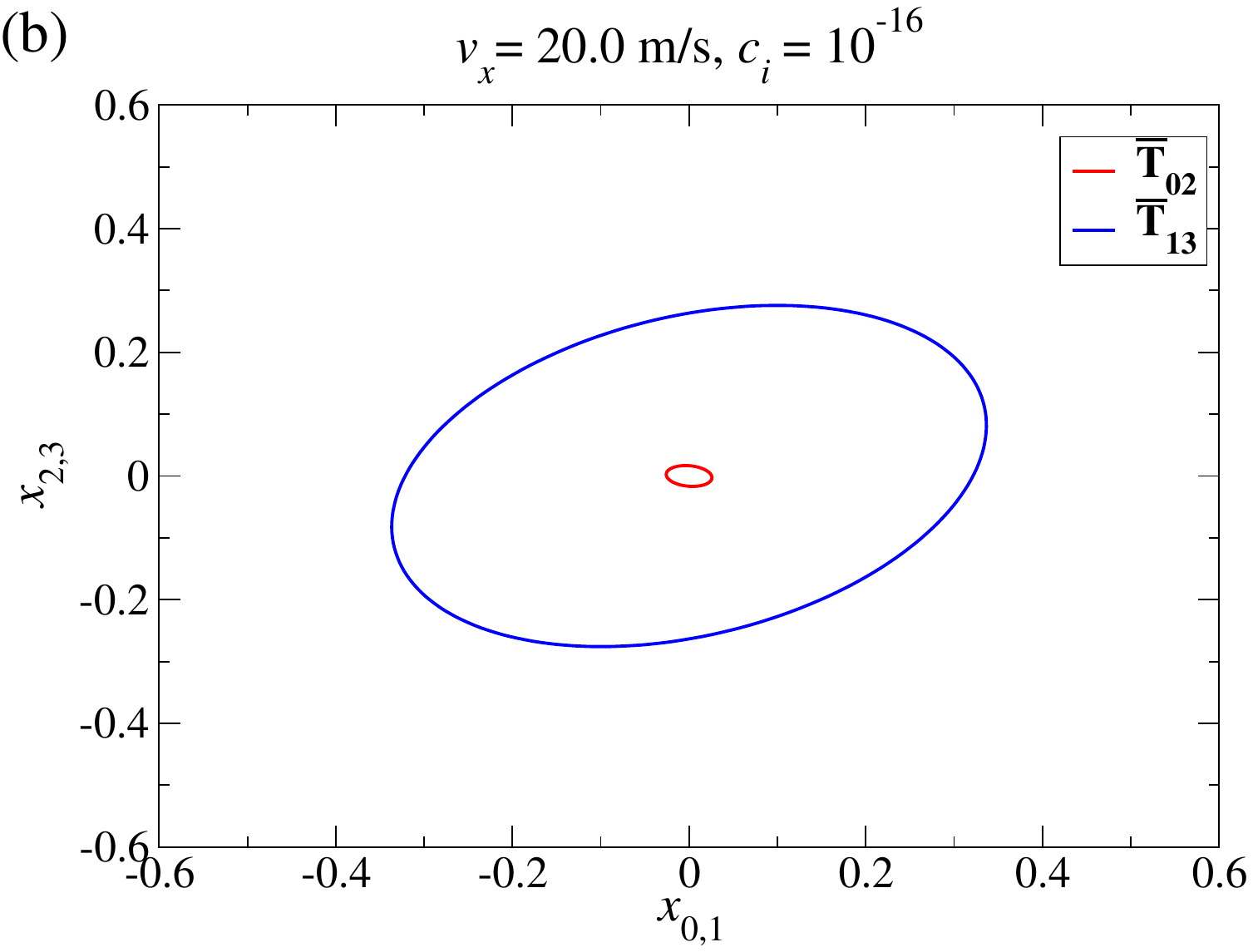}}%
{\includegraphics*[width=0.25\linewidth]{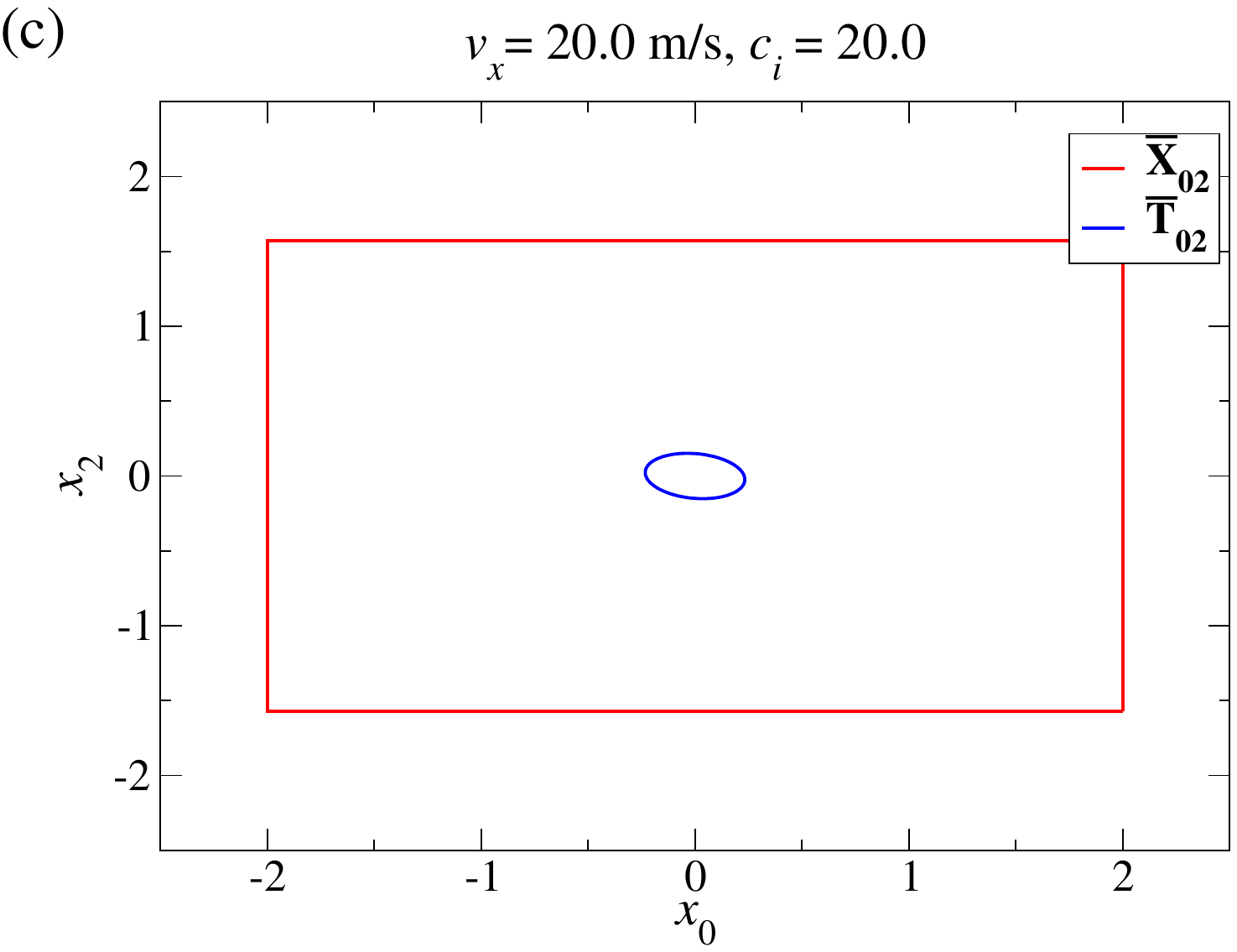}}%
{\includegraphics*[width=0.25\linewidth]{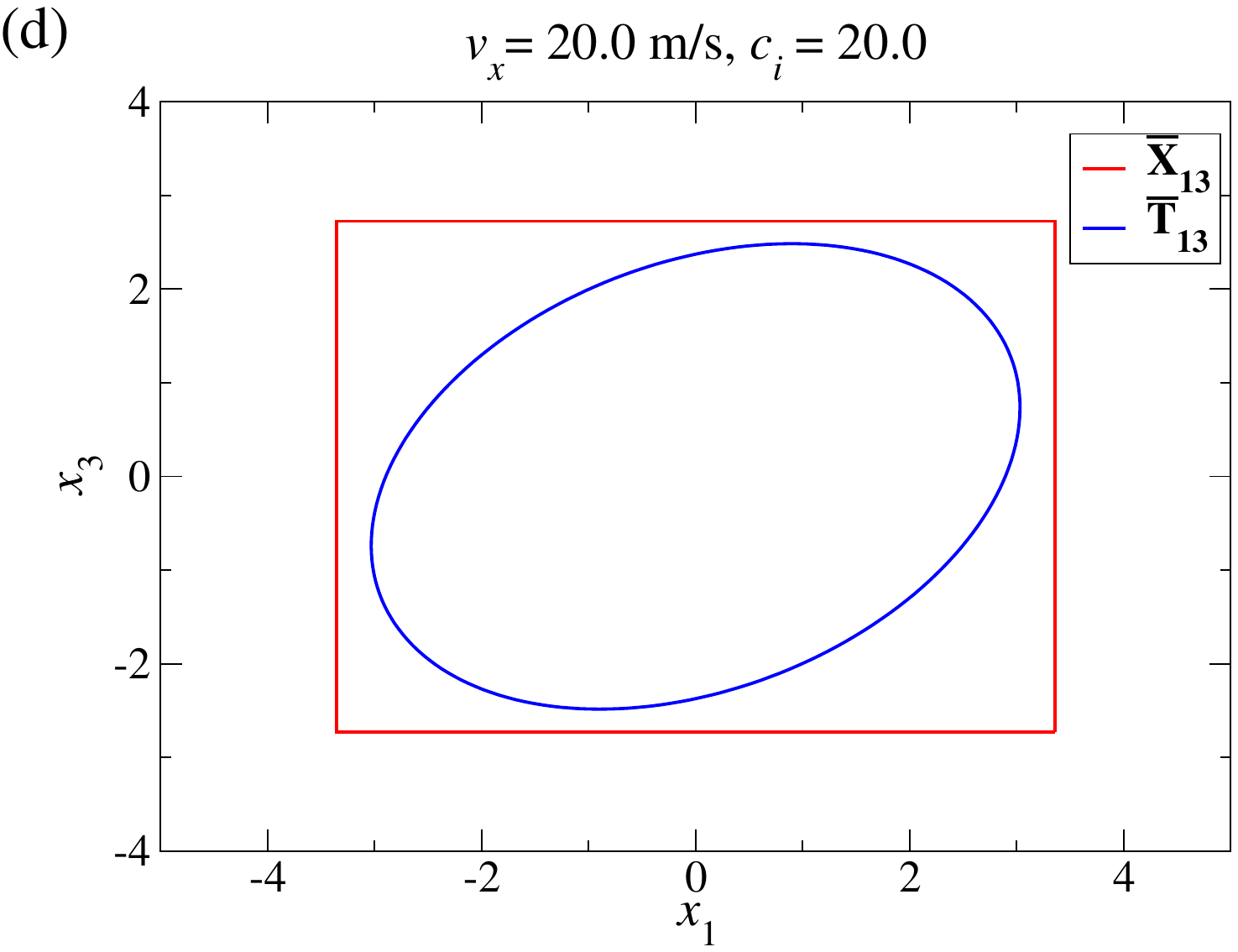}}%
\caption{Outer boundaries of  sets $\mathbb{\bar T}_{02}$ and $\mathbb{\bar T}_{13}$ when $v_x$ = 20.0 or 21.1 m/s and $c_i$ = $10^{-16}$ or 20.0. Here, $x_0  \equiv  \Delta $, $x_1  \equiv \dot \Delta $, $x_2  \equiv \theta $, and $x_3  \equiv \dot \theta $.}
\end{figure*}

The relevant nominal state and observer dynamics are expressed as follows, respectively:
\begin{subequations}
\begin{equation}
{\bf \bar x'}_{i + 1}  = {\bf A' \bar x'}_i  + {\bf B' \bar u'}_i  ,
\end{equation}
\begin{equation}
{\bf \hat x'}_{i + 1}  = {\bf A'\hat x'}_i  + {\bf B'u'}_i  + {\bf L'}\left( {{\bf \hat y'}_i  - {\bf C'\hat x'}_i } \right).
\end{equation}
\end{subequations}
The actual system state ${\bf x'}$ can be represented using the nominal state ${\bf \bar x'} $,  estimation error ${\bf e'}$, and prediction error ${\bf \xi '}$ as follows:
\begin{equation}
{\bf x'}  = {\bf \bar x'}  + {\bf e'}  + {\bf \xi '}.
\end{equation}
The parameters  ${\bf e'}$ and ${\bf \xi '}$ are expressed as follows: 
\begin{subequations}
\begin{equation}
{\bf e'}  = {\bf x'}  - {\bf \hat x'},
\end{equation}
\begin{equation}
{\bf \xi '}  = {\bf \hat x'}  - {\bf \bar x'}  .
\end{equation}
\end{subequations}
These parameters satisfy the following uncertain difference equations:
\begin{subequations}
\begin{equation}
{\bf e'}_{i{\rm  + }1}  = \left( {{\bf A'} - {\bf L'C'}} \right){\bf e'}_i  + {\bf v}^{x'}  - {\bf L'v}^{y'} ,
\end{equation}
\begin{equation}
{\bf \xi '}_{i + 1}  = \left( {{\bf A'} + {\bf B'K'}} \right){\bf \xi '}_i  + {\bf L'}\left( {{\bf C'e'} + {\bf v}^{y'} } \right).
\end{equation}
\end{subequations}
The robust positive-invariant sets $\mathbb{S}^e$ and $\mathbb{S}^x $ satisfy the following expressions, respectively:
\begin{subequations}
\begin{equation}
  \left( {{\bf A'} - {\bf L'C'}} \right)\mathbb{S}^e  \oplus \mathbb{V}^{x'} \ominus {\bf L'}\mathbb{V}^{y'}  \subseteq \mathbb{S}^e,
\end{equation}
\begin{equation}
 \left( {{\bf A'} + {\bf B'K'}} \right)\mathbb{S}^x   \oplus  {\bf L'}\left( {\bf C'}\mathbb{S}^e   \oplus \mathbb{V}^{y'}  \right) \subseteq \mathbb{S}^x . 
\end{equation}
\end{subequations}
The nominal state ${\bf \bar x'}$ and  control input ${\bf \bar u'}$ are then subject to the tighter constraints $\mathbb{ \bar X'}$ and $\mathbb{\bar U'}$, which are defined as follows, respectively:
\begin{subequations}
\begin{equation}
\mathbb{\bar X'} =  \mathbb{ X'} \ominus \left( {\mathbb{S}^e  \oplus \mathbb{S}^x } \right) ,
\end{equation}
\begin{equation}
\mathbb{\bar U'} = \mathbb{U'}  \ominus {\bf K'}\mathbb{S}^x .
\end{equation}
\end{subequations}

Fig. 2 presents the results obtained using the aforementioned equations for sets $\mathbb{X}'$, $\mathbb{\bar X}'$, $\mathbb{S}^e$,  and $\mathbb{S}^x$ when $v_x$ = 21.1 or 20.0 m/s and $\kappa $ = $\pm$0.05 or $\pm$0.1 1/m. In this figure, the size of the nominal constraint region $\mathbb{\bar X}'$ increases when that of the $\mathbb{S}^e \oplus \mathbb{S}^x$ region decreases. Consequently, the largest set of feasible nominal states is obtained when $v_x$ = 20.0 m/s and $\kappa $ = $\pm$0.05 1/m [Fig. 2(d)]. To reduce the online computation burden, the computation expressed in (21), which is based on tightened constraints, can be performed offline, and a lookup table can be established for online use. In the proposed approach \cite{Lee26}, the tightened constraints are  adaptively estimated online on the basis of the road curvature $\kappa $ detected online by a vision-based perception system. Compared with the traditional tube-based approach, in which tightened constraints are derived from the worst-case disturbance, the approach proposed in \cite{Lee26} can reduce  system conservativeness and enhance controller performance through the adaptive use of tightened constraints.

The OF-MPC formulation presented in the previous section can be extended to incorporate robustness, thereby developing the proposed OF-MPC\ framework. Thus, the target calculation problem is reformulated as Problem P1' = P1(${\bf \hat d}^ +$, ${\bf r}_{sp}$, $\mathbb{\bar X}$/$\mathbb{\bar U}$), where $\mathbb{\bar U} \equiv \mathbb{\bar U'}$. The second and fourth components of $\mathbb{\bar X}$ are replaced with the corresponding components from  $ \mathbb{\bar X'}$, whereas the original constraints are retained for the other components. Solving Problem P1' yields ${\bf \bar x}_s$ and ${\bf \bar u}_s$, which are robust target state and control vectors, respectively. The regulation problem is then reformulated as Problem P2' = P2(${\bf \hat x}^ +$, ${\bf \hat d}^ +$, ${\bf \bar x}_s$/${\bf \bar u}_s$, $\mathbb{\bar X}$/$\mathbb{\bar T}$/$\mathbb{\bar U}$), where $\mathbb{\bar T}$ and $\mathbb{\bar T'}$ have a similar relationship to that between  $\mathbb{\bar X}$ and $\mathbb{\bar X'}$. The construction of set $\mathbb{\bar T}$ is described as follows.

Herein, a positively invariant ellipsoidal level set  \cite{Raw26, Raw17} is used as the terminal set because of its relatively low computational complexity and suitability for high-dimensional systems \cite{May13}. Therefore, the terminal set is represented as follows: 
\begin{equation}
\mathbb{\bar T }_i:= \left\{ {\left. {\bf \bar x} \right|\left( {\bf \bar x} - {\bf \bar x_s} \right)^T {\bf P}\left( {{\bf \bar x} - {\bf \bar x_s} } \right) \le  c _i} \right\}.
\end{equation}
The ellipsoid scaling constraint $c_i$ is an offline design parameter that determines the size of the constraint region and can be selected to be close to machine precision \cite{Raw26}. Because ${\bf \bar x_s}$ varies with ${\bf \hat d}$, $\mathbb{\bar T}_i$ also varies with  ${\bf \hat d}$. Such variations in $\mathbb{\bar T}_i$ may violate the recursive feasibility of Problem P2' \cite{Raw17}. To address this limitation, an enlarged ellipsoidal set  $\mathbb{\bar T }$ is adopted as the terminal set, with each $\mathbb{\bar T }_i$ set incorporated into $\mathbb{\bar T }$. The relevant computation is described as follows.

To analyze the influence of variations in the estimated disturbance ${\bf \delta \hat d}$ on the steady-state target within a single MPC step, the mapping ${\bf \delta \hat d} \mapsto {\bf \delta \bar z}_s $, where ${\bf \bar z}_s  \equiv \left[ {{\bf \bar x}^T_s ,{\bf \bar u}^T_s } \right]^T$, is established on the basis of the following relationship:
\begin{equation}
{\bf \delta \bar z}_s  = {\bf \Pi}  {\bf \delta \hat d},
\end{equation}
where $\bf \Pi$ is the steady-state mapping matrix. This matrix is expressed as follows:
\begin{equation}
{\bf \Pi}  = \left[ {\begin{array}{*{20}c}
   {{\bf I} - {\bf A}} & { - {\bf B}}  \\
   {\bf HC} & {\bf 0}  \\
\end{array}} \right]^{ - 1}\left[ {\begin{array}{*{20}c}
   {{\bf B}_d }  \\
   { {\bf r}_{sp} - {\bf HC}_d }  \\
\end{array}} \right].
\end{equation}
Here, the mapping is assumed to be uniformly continuous. Thus, ${\bf \delta \bar z}_s$ is bounded by the Lipschitz constant $\left\| \Pi  \right\|_2$ as follows:
\begin{equation}
\left\| {{\bf \delta \bar z}_s } \right\| = \left\| {\Pi {\bf \delta \hat d}} \right\| \le \left\| \Pi  \right\|_2 \left\| {{\bf \delta \hat d}} \right\| \le \left\| \Pi  \right\|_2 \left\| {{\bf \delta \hat d}} \right\|_{\max } ,
\end{equation}
where
\begin{equation}
\left\| {{\bf \delta \hat d}} \right\|_{\max }  = \sqrt {\left( {2 \hat d_{0,\max } } \right)^2  + \left( {2 \hat d_{1,\max } } \right)^2 }  = 2\kappa _{\max } \sqrt {\gamma _0^2  + \gamma _1^2 } .
\end{equation}
Because terminal set $\mathbb{\bar T}$ is only related to the state variables, the varying state target $\left\| {{\bf \delta \bar x}_s } \right\|$ can be computed using ${\bf \Pi} _x$, which is a submatrix of ${\bf \Pi} $ comprising the rows about ${{\bf  \bar x}_s }$, as follows: 
\begin{equation}
\left\| {{\bf \delta \bar x}_s } \right\| \le \left\| \Pi _x \right\|_2 \left\| {{\bf \delta \hat d}} \right\|_{\max } \equiv \left\| {\bf \delta \bar x}_s \right\|_{\max }.
\end{equation}

The enlarged ellipsoidal set $\mathbb{\bar T }$ can now be defined as follows:
\begin{equation}
\mathbb{\bar T }:= \left\{ {\left. {\bf \bar x} \right| {\bf \bar x} ^T {\bf P} {{\bf \bar x} } \le  c _f} \right\},
\end{equation}
where the scaling parameter $c_f$ is obtained under the maximum disturbance variation. This parameter is expressed as follows:
\begin{equation}
c_f  = \left( {\sqrt {c_i }  + \left\| {\bf \delta \bar x}_s \right\|_{\max } } \right)^2 .
\end{equation}
According to this equation, the set inclusion property $\mathbb{\bar T }({\bf \bar x}_s ({\bf \hat d})) \subseteq \mathbb{\bar T }({\bf \bar x}_s ({\bf \hat d} + {\bf \delta \hat d}))$ $\forall {\bf \hat d} \in \mathbb{W}$ is ensured, which guarantees the recursive feasibility  of the MPC problem in the OF-MPC framework. The proof of the closed-loop properties of the MPC problem in the proposed OF-MPC framework is provided in the Appendix.

\begin{figure*}[!t]
{\includegraphics*[width=0.33\linewidth]{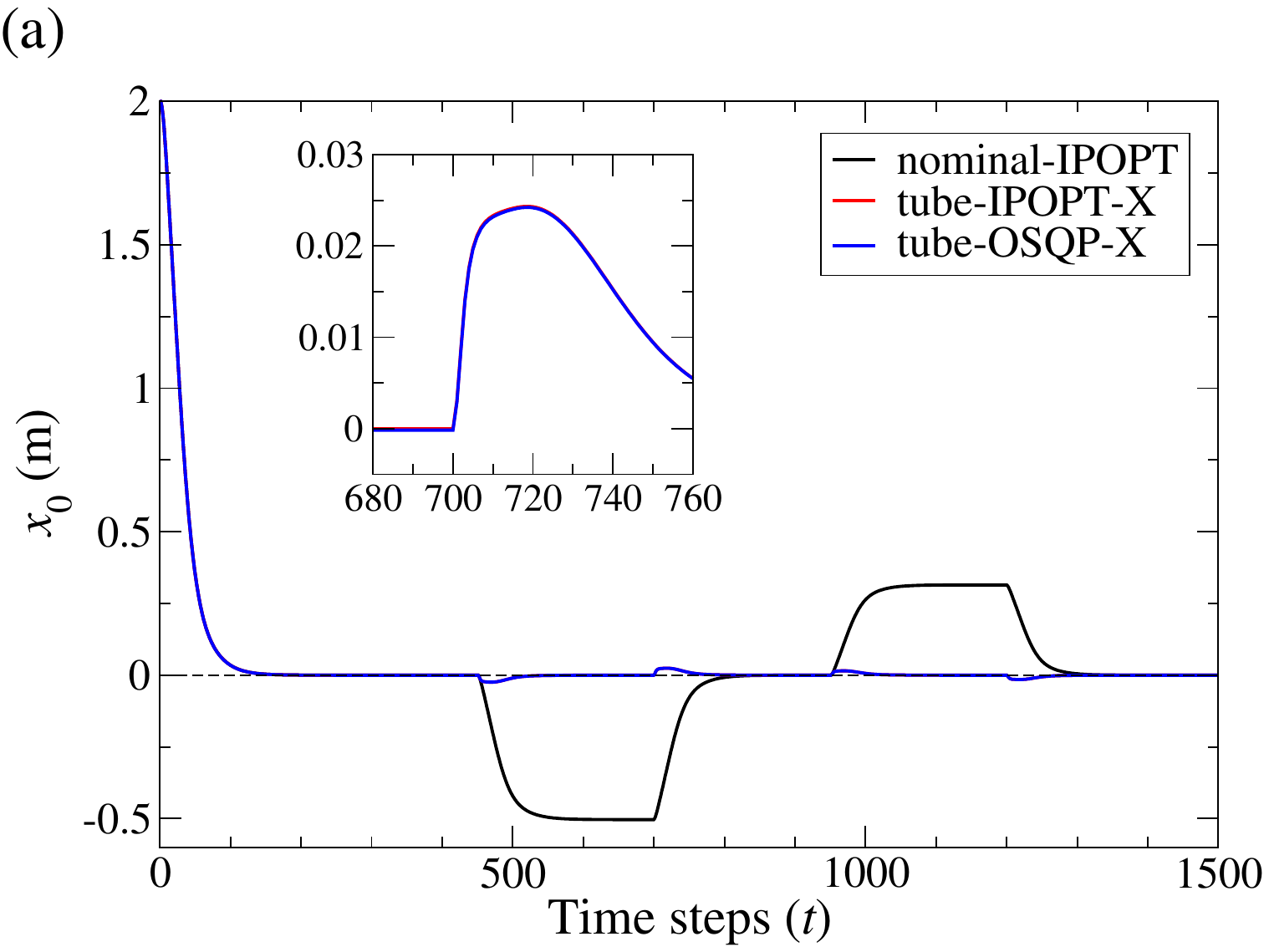}}%
{\includegraphics*[width=0.33\linewidth]{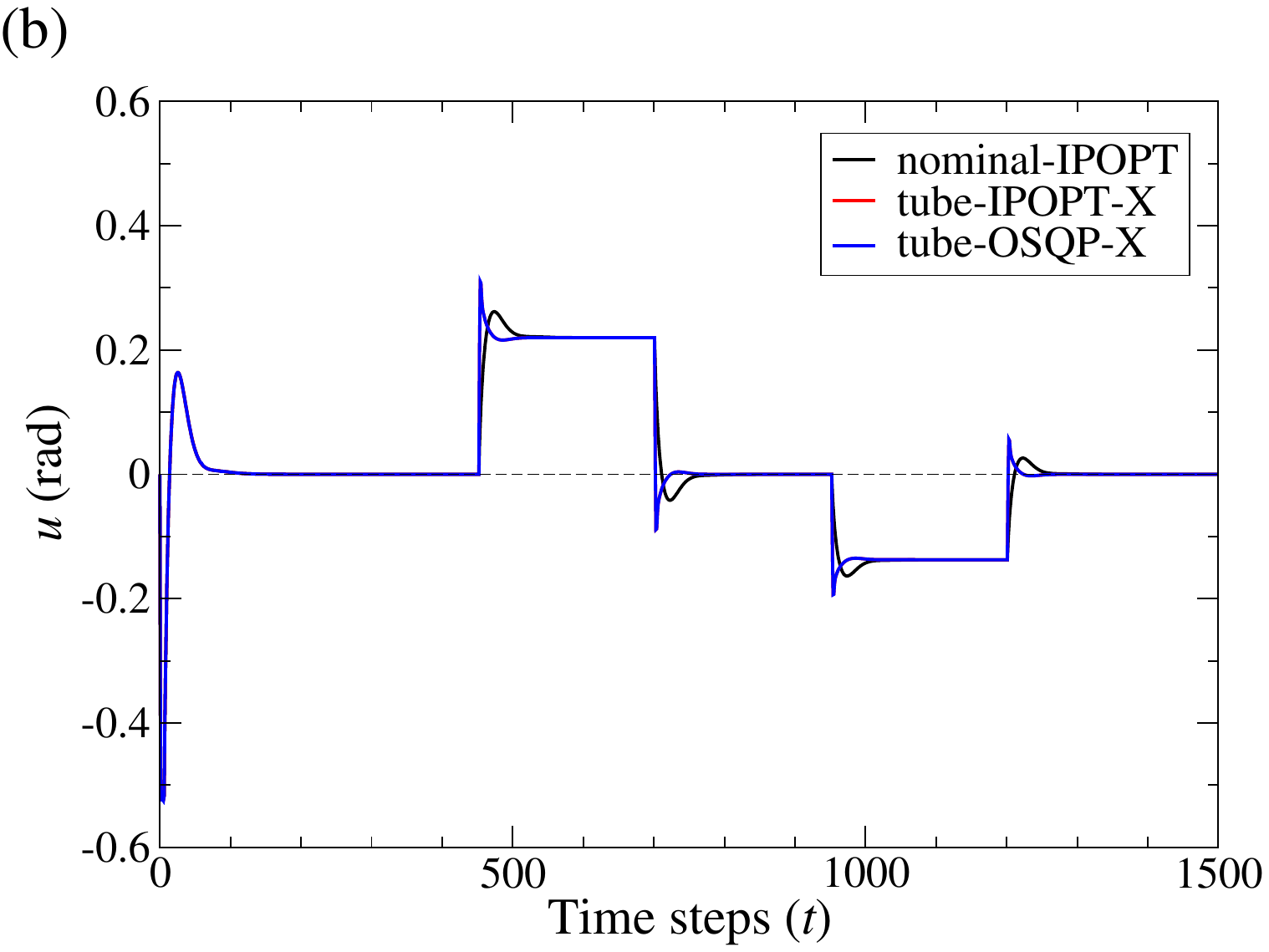}}%
{\includegraphics*[width=0.33\linewidth]{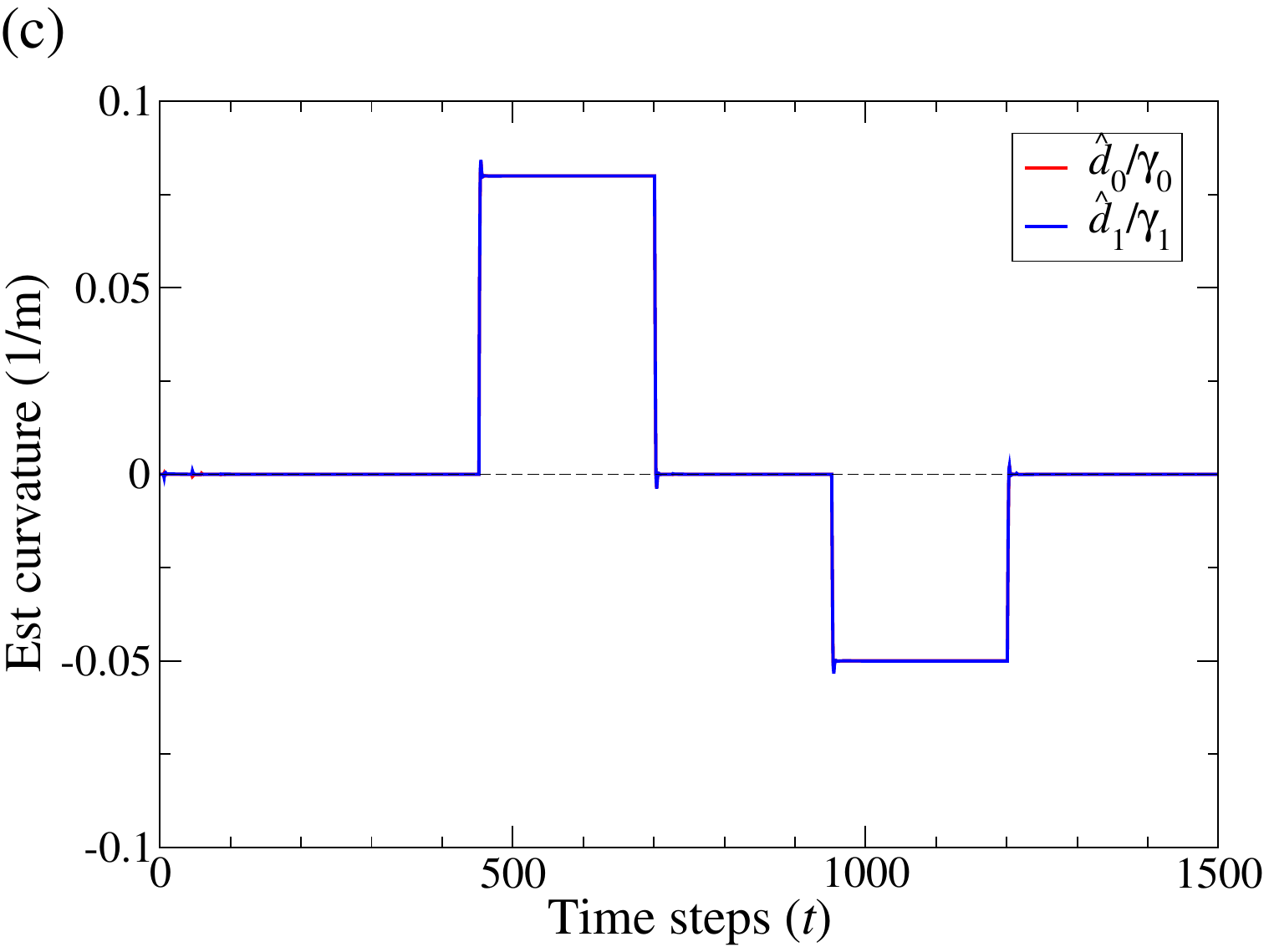}}%
\caption{Numerical simulation results obtained with different IPOPT- and OSQP-based solvers under tightened constraints  when $v_x$ = 20.0 m/s, $\left| {x_1 } \right| \le$ 13.0 m/s, and $\left| {x_3 } \right| \le$ 10.0 rad/s: (a) trajectory of the actual system state $x_0$ ($\Delta$); (b) trajectory of the system control input $u$ = $\delta$; and (c) trajectories of the estimated road curvatures obtained with the tube-OSQP-X algorithm, which were obtained by dividing the estimated disturbance vectors ${ \hat d}_0$ and ${ \hat d}_1$ by $\gamma _0$ and $\gamma _1$, respectively.} 
\end{figure*}

\begin{figure*}[!t]
{\includegraphics*[width=0.33\linewidth]{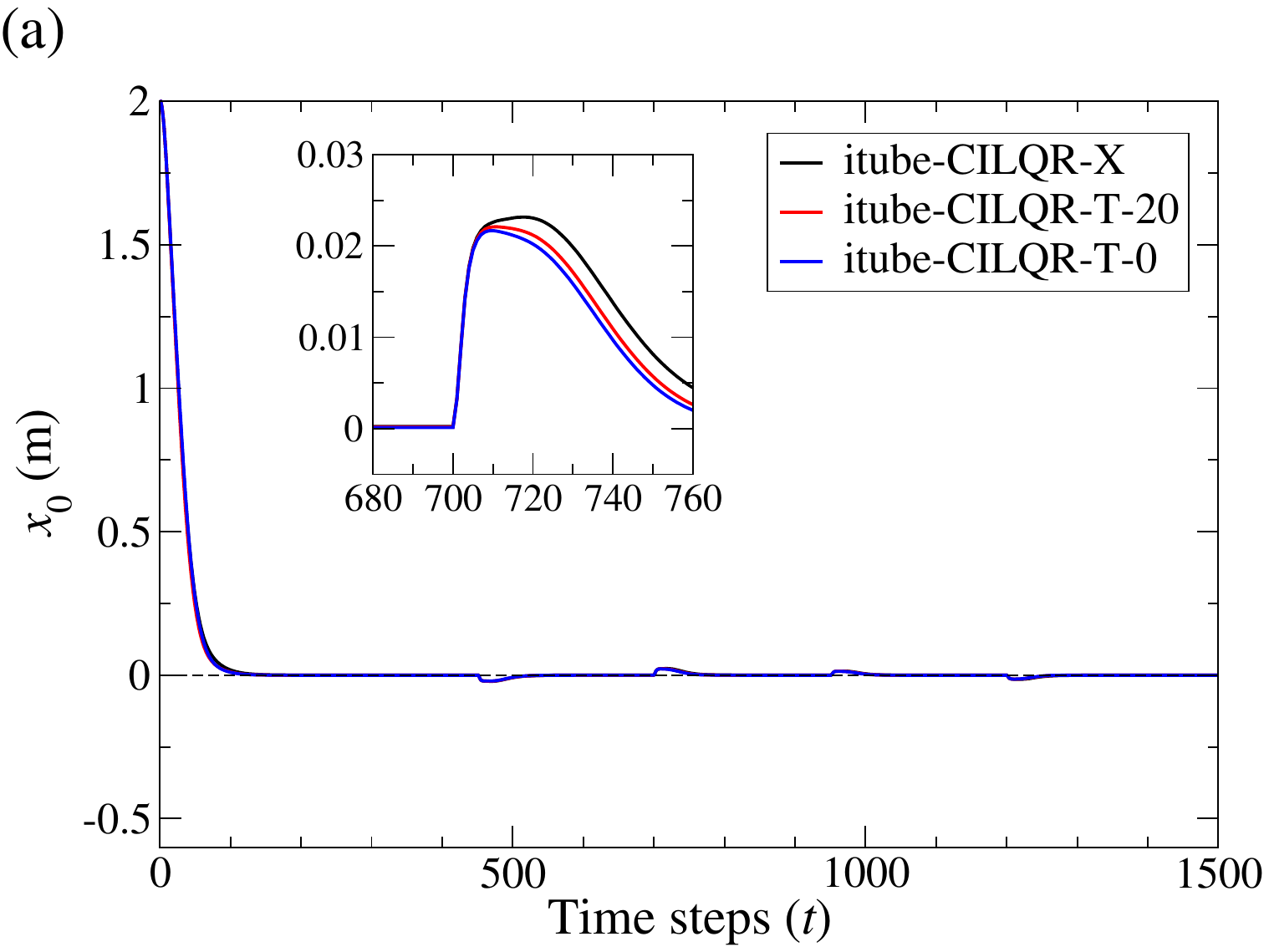}}%
{\includegraphics*[width=0.33\linewidth]{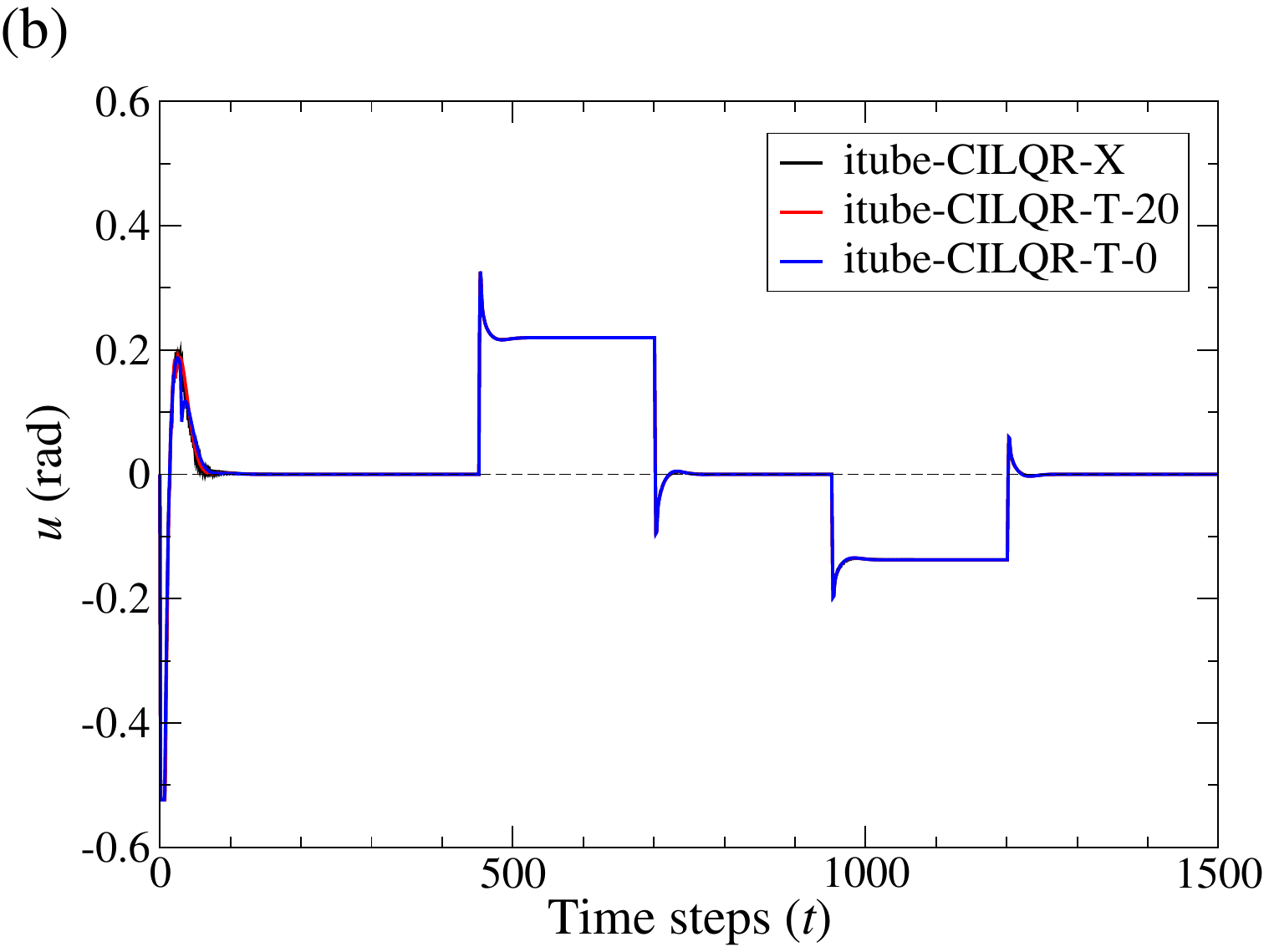}}%
{\includegraphics*[width=0.33\linewidth]{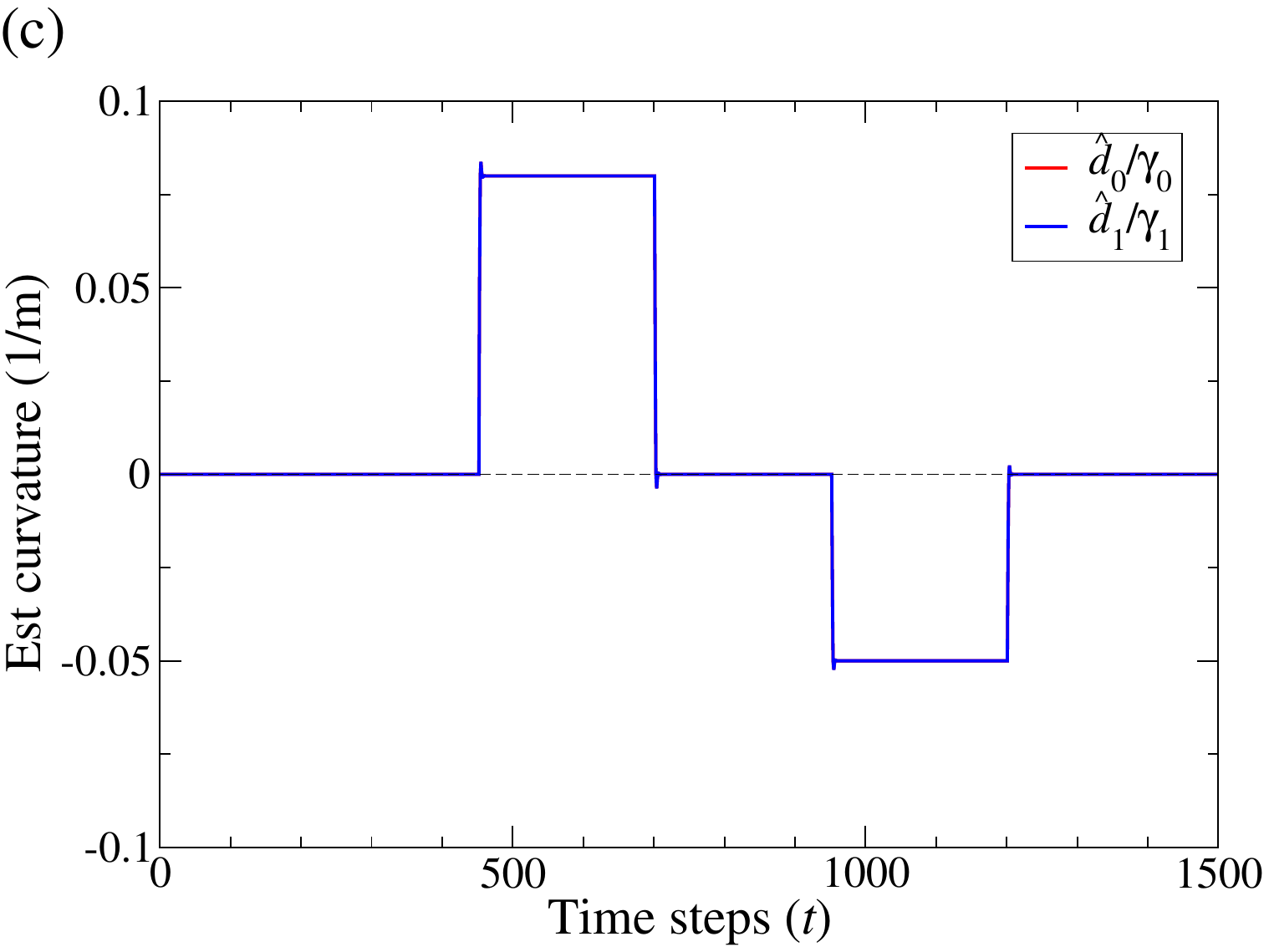}}%
\caption{Numerical simulation results obtained with different itube-CILQR-based solvers when $v_x$ = 20.0 m/s, $\left| {x_1 } \right| \le$ 13.0 m/s, and $\left| {x_3 } \right| \le$ 10.0 rad/s: (a) trajectory of the actual system state $x_0$ ($\Delta$); (b) trajectory of the system control input $u$ = $\delta$; and (c) trajectories of the estimated road curvatures  obtained with the itube-CILQR-T-0 algorithm, which were obtained by dividing the estimated disturbance vectors ${ \hat d}_0$ and ${ \hat d}_1$ by $\gamma _0$ and $\gamma _1$, respectively.} 
\end{figure*}

Fig. 3 displays the results computed for the enlarged ellipsoidal set when $v_x$ = 21.1 or 20.0 m/s and $c_i$ = ${10}^{-16}$ or 20.0. In this figure, $\mathbb{\bar T}_{02}$ and  $\mathbb{\bar T}_{13}$ represent components extracted from $\mathbb{\bar T}$ through projection. In general, all ellipsoidal terminal sets in Fig. 3 are contained within the corresponding state constraint regions (e.g., $\mathbb{\bar X}'$ in Fig. 2). The constraint region of $\mathbb{\bar T}_{13}$ (in blue) is larger than that of  $\mathbb{\bar T}_{02}$ (in red) because the influences of disturbance (${\bf \hat d}$) uncertainty on  $\dot \Delta$ and $\dot \theta$ are stronger for  $\mathbb{\bar T}_{13}$. Moreover, the constraint region for the high-speed case [Fig. 3(a)] is larger than that for the low-speed case [Fig. 3(b)] because the disturbance uncertainty is higher in the high-speed case. When $c_i$ is increased to 20.0, the outer boundary of $\mathbb{\bar T}_{13}$ approaches that of $\mathbb{\bar X}_{13}$ if $v_x$ = 20.0 m/s [Fig. 3(d)]. Under all conditions shown in Fig. 3, $\mathbb{\bar T}_{02}$ occupies a smaller region than does $\mathbb{\bar X}_{02}$, which may compromise the convergence of the MPC solver, as explained in Section V. In addition, the terminal ellipsoid constraints are nonlinear, which can result in a higher online computational burden than that observed with linear constraints. Thus, box constraints derived from the largest rectangular constraint regions of $\mathbb{\bar T}_{02}$ and $\mathbb{\bar T}_{13}$ are used in the online optimization \cite{Lee26}. The sizes of these rectangular regions are influenced by the predefined disturbance bounds ($\mathbb W$) and can thus be adjusted through appropriate parameter selection.

The proposed robust OF-MPC framework can be implemented using  standard software packages to solve the regulation problem. It can also employ the itube-CILQR algorithm to solve the corresponding tube-based MPC problem. This algorithm has a flexible architecture, thus facilitating the evaluation of different control formulations. The itube-CILQR solver is described in the following section.

\section{Itube-CILQR Algorithm}
This section describes the itube-CILQR algorithm, which was proposed in our previous study \cite{Lee26}  and improved in the present study. On the basis of the CILOR algorithm \cite{Che17, Che19, Liu24}, Problem P2' is converted into an unconstrained optimization problem (Problem P3) by using convex exponential barrier functions. Furthermore, additional interpolation variables are introduced into the itube-CILQR algorithm to adaptively adjust the size of the tube cross section according to the detectable disturbance  $\kappa $. By incorporating the aforementioned adaptive mechanism, the improved itube-CILQR algorithm achieves lower conservatism and better control performance compared with classical tube-based algorithms at the cost of an acceptable increase in computation time \cite{Lee26}. Problem P3, which incorporates the interpolation variable sequence ${\bf L } \equiv \left\{ {{\bf \Lambda }_0 ,{\bf \Lambda }_1 ,...,{\bf \Lambda }_N } \right\}$ and the interpolation variable  vector ${\bf \Lambda } \equiv \left[ {\begin{array}{*{20}c}
   {\underline {\lambda} } & {\lambda } & {\overline {\lambda} }  \\
\end{array}} \right]^T $,  is expressed as follows:

\begin{figure*}[!t]
{\includegraphics*[width=0.33\linewidth]{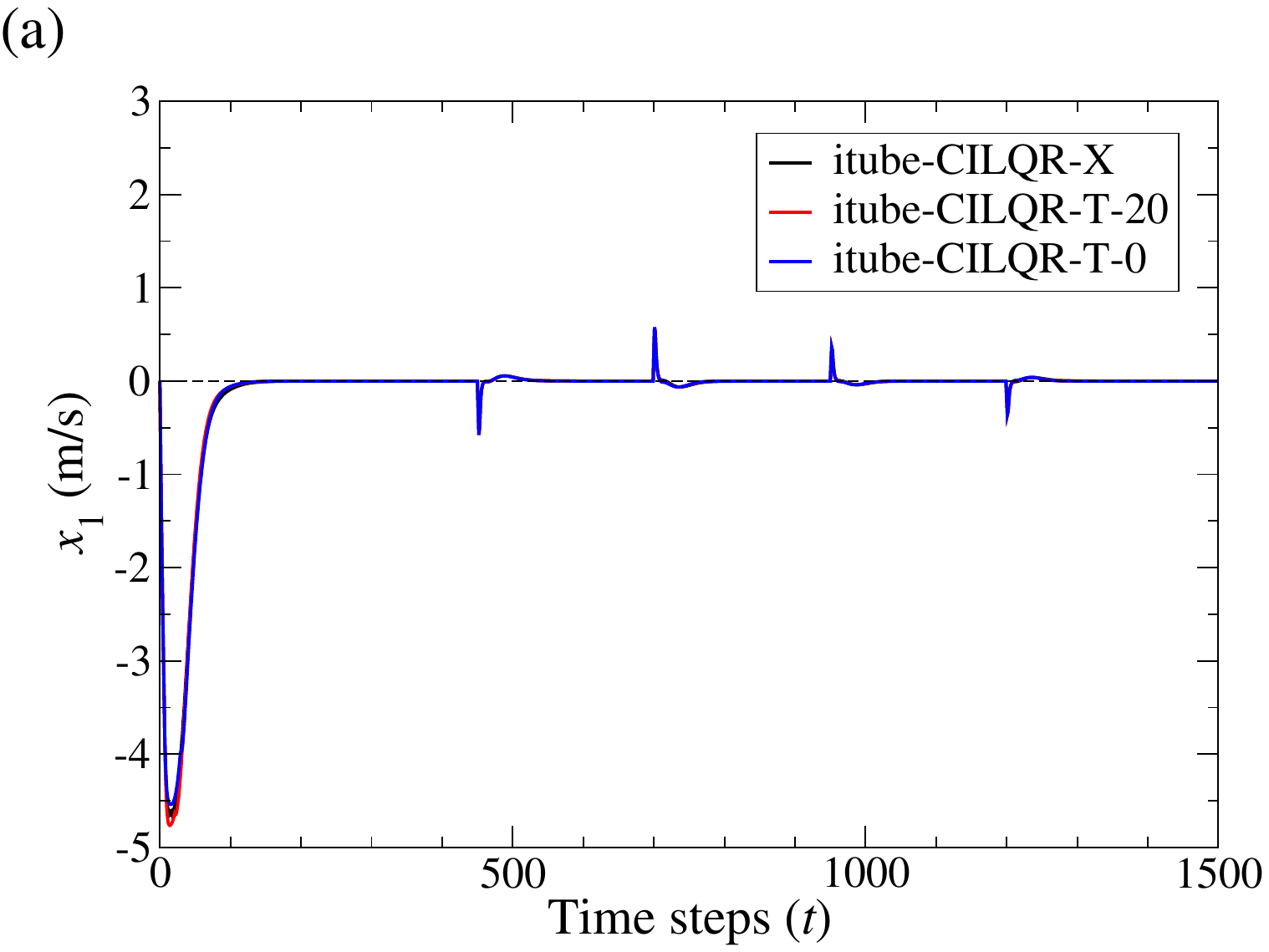}}%
{\includegraphics*[width=0.33\linewidth]{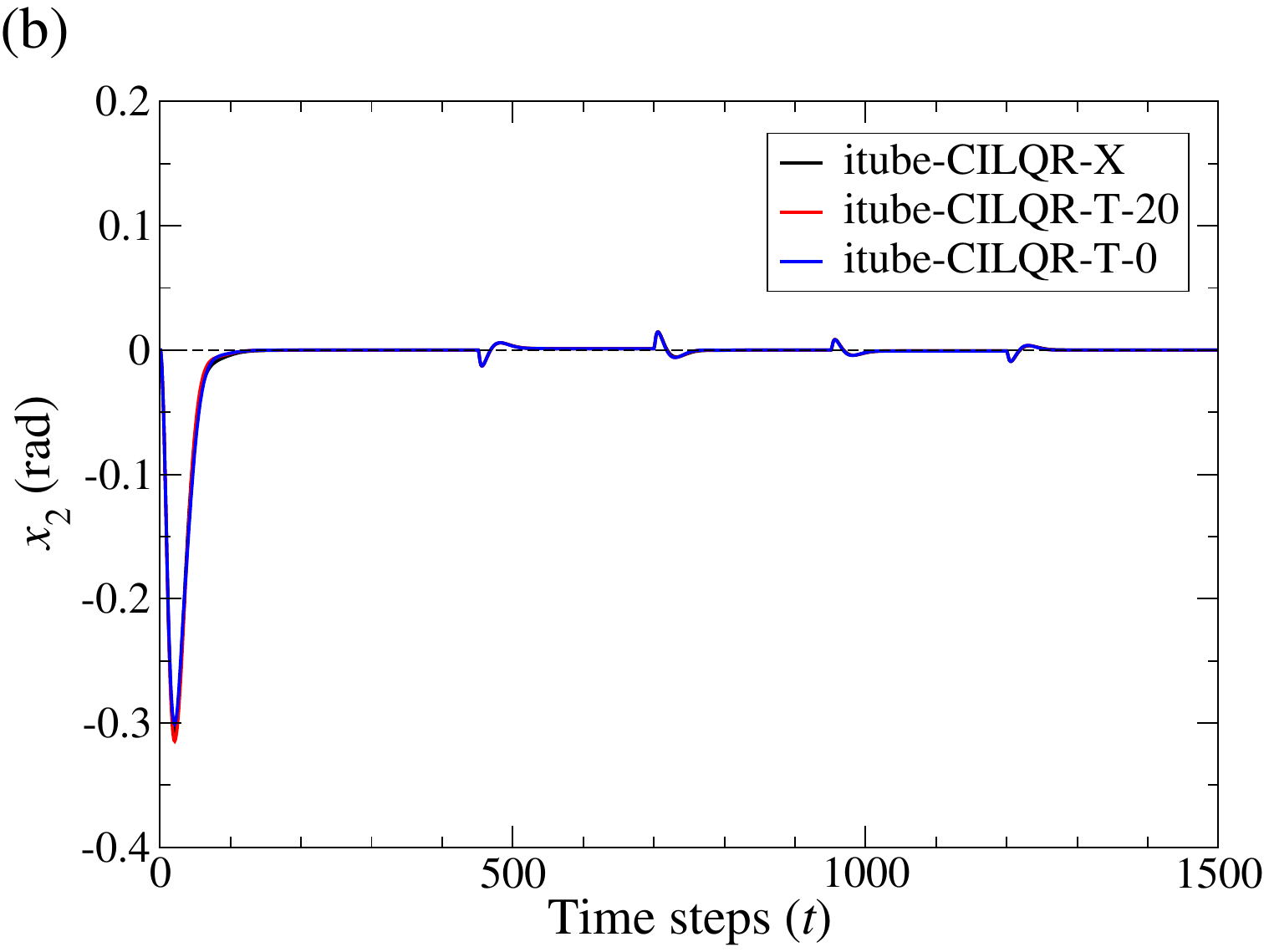}}%
{\includegraphics*[width=0.33\linewidth]{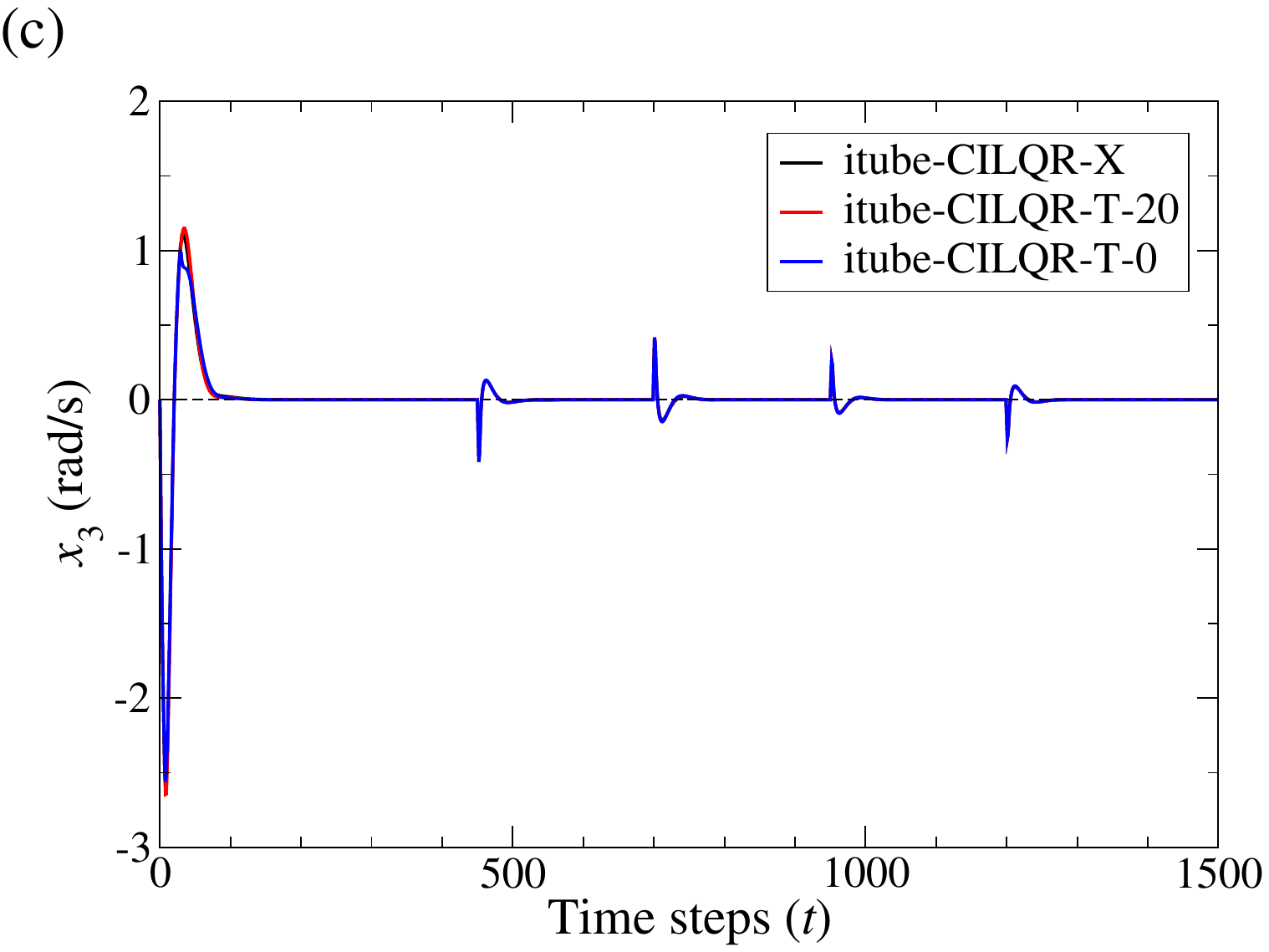}}%
\caption{Numerical simulation results obtained with the itube-CILQR solver when $v_x$ = 20.0 m/s, $\left| {x_1 } \right| \le$ 13.0 m/s, and $\left| {x_3 } \right| \le$ 10.0 rad/s: trajectories of the actual system states (a) $x_1$, (b) $x_2$, and (c) $x_3$. } 
\end{figure*}

\begin{table*}[!t]
\caption{Maximum Lateral Offsets and Average Computation Times Achieved by the tube-IPOPT-X, tube-OSQP-X, itube-CILQR-X, itube-CILQR-T-20, and itube-CILQR-T-0 Algorithms$^{a}$}
\begin{center}
\begin{tabular}{cccccc}\Xhline{1.5pt} 
Method &tube-IPOPT-X$^{b}$ &tube-OSQP-X$^{b}$ & itube-CILQR-X &  itube-CILQR-T-20&  itube-CILQR-T-0\\ \hline
Max. $\Delta$ (m)&0.024333 & 0.024268   & 0.023162  & 0.022097 & 0.021675\\  
Latency (ms)&6.45 & 0.84  & 1.10  & 4.21 &  7.35 \\  \Xhline{1.5pt} 
\multicolumn{6}{l}{$^{a}$\scriptsize{All algorithms were implemented in C++.}} \\
\multicolumn{6}{l}{$^{b}$\scriptsize{The tube-IPOPT-T-0/20 and tube-OSQP-T-0/20 algorithms could not achieve convergence  for the considered optimization problem.}} 
\\
\end{tabular}
\end{center}
\end{table*}

Problem P3:
\begin{subequations}
\begin{equation}
\mathop {\min }\limits_{{\bf \bar X},{\bf \bar U},{\bf L}}   \bar{J}_s  + \bar{J}_t  + J_l + J_{x0}+ J_{x1}+ J_{u}+ J_{lu}+ J_{ls},
\end{equation}
\begin{equation}
\begin{split}
\bar{J}_s  &= \sum\limits_{i = 0}^{N - 1} { \left( {{\bf \bar x}_i  - {\bf \bar x}_s } \right)^T {\bf Q}\left( {{\bf \bar x}_i  - {\bf \bar x}_s } \right) }\\
&+ \sum\limits_{i = 0}^{N - 1} { \left( {{\bf \bar u}_i  - {\bf \bar u}_s } \right)^T {\bf R}\left( {{\bf \bar u}_i  - {\bf \bar u}_s } \right) },
\end{split}
\end{equation}
\begin{equation}
\bar{J}_t  = \left( {{\bf \bar x}_N  - {\bf \bar x}_s } \right)^T {\bf P }\left( {{\bf \bar x}_N  - {\bf \bar x}_s } \right)
\end{equation}
\begin{equation}
 J_l  = \sum\limits_{i = 0}^N {{\bf \Lambda }_i^T {\bf S\Lambda }_i } ,
\end{equation}
\begin{equation}
\begin{split}
J_{x0}  &= \\
&q_s \sum\limits_{k \in \left\{ {0,2} \right\}} {\sum\limits_{i = 0}^{N-1} {\exp \left( {\underline{\bar x}_{k }  - \bar x_{k,i} } \right) + \exp \left( {\bar x_{k,i}  - \overline{\bar x}_{k} } \right)} } \\
&+{\exp \left( {\underline{\bar t}_{k}  - \bar x_{k,N} } \right) + \exp \left( {\bar x_{k,N}  - \overline{\bar t}_{k } } \right)},
\end{split}
\end{equation}
\begin{equation}
\begin{split}
J_{x1} & = \\
q_s & \sum\limits_{k \in \left\{ {1,3} \right\}}  \sum\limits_{i = 0}^{N-1} { \exp  { \left[ { - \left( {\underline{\lambda}_{i} \underline{\bar x}_{k}  + \lambda _{i} \bar x_{det,k }  + \overline {\lambda}_{i} \overline{\bar x}_{k} } \right) - \bar x_{k,i} } \right]} }  \\
&+  \exp  { \left[ {\bar x_{k,i}  - \left( {\underline{\lambda}_{i} \underline{\bar x}_{k}  + \lambda _{i} \bar x_{det,k}  + \overline {\lambda}_{i} \overline{\bar x}_{k} } \right)} \right]} \\
&+  \exp  { \left[ { - \left( {\underline{\lambda}_{N} \underline{\bar t}_{k}  + \lambda _{N} \bar t_{det,k}  + \overline {\lambda}_{N} \overline{\bar t}_{k} } \right) - \bar x_{k,N} } \right]} \\
&+  \exp  { \left[ {\bar x_{k,N}  - \left( {\underline{\lambda}_{N} \underline{\bar t}_{k}  + \lambda _{N} \bar t_{det,k}  + \overline {\lambda}_{N} \overline{\bar t}_{k} } \right)} \right]} ,
\end{split}
\end{equation}
\begin{equation}
\begin{split}
J_{u}  &= q_{u}\sum\limits_{i = 0}^{N - 1} { \exp  { \left[ { - \left( {\underline{\lambda}_{i} \underline{\bar \delta}  + \lambda _{i} \bar \delta _{det}  + \overline {\lambda}_{i} \overline{\bar \delta} } \right) - \bar\delta _i } \right]} } \\
&+  \exp  { \left[ {\bar\delta _i  - \left( {\underline{\lambda}_{i} \underline{\bar \delta}  + \lambda _{i} \bar \delta _{det }  + \overline {\lambda}_{i} \overline{\bar \delta} } \right)} \right]} ,
\end{split}
\end{equation}
\begin{equation}
\begin{split}
J_{lu}  &= q_{l1} \sum\limits_{i = 0}^N {\exp \left( { - q_{l1} \underline{\lambda}_i } \right) + \exp \left[ {q_{l1} \left( {\underline{\lambda}_i  - 1.0} \right)} \right]}  \\ &+ \exp \left( { - q_{l1} \lambda _i } \right) + \exp \left[ {q_{l1} \left( {\lambda _i  - 1.0} \right)} \right] \\ &+ \exp \left( { - q_{l1} \overline{\lambda}_i  } \right) + \exp \left[ {q_{l1} \left( {\overline{\lambda}_i   - 1.0} \right)} \right],
\end{split}
\end{equation}
\begin{equation}
\begin{split}
J_{ls} & = q_{l2} \sum\limits_{i = 0}^N {\exp \left\{ {q_{l2} \left[ {1.0 - \left( {\underline{\lambda}_{i}  + \lambda _{i}  + \overline {\lambda}_{i} } \right)} \right]} \right\}}  \\
&+ \exp \left\{ {q_{l2} \left[ {\left( {\underline{\lambda}_{i}  + \lambda _{i}  + \overline {\lambda}_{i} } \right) - 1.0} \right]} \right\}
\end{split}
\end{equation}
\end{subequations}
subject to
\begin{subequations}
\begin{equation}
{\bf \bar x}_{i + 1}  = {\bf \bar f}\left( {{\bf \bar x}_i ,{\bf \bar u}_i,{\bf \hat d}^ + } \right)  ,\quad 0 \le i < N,
\end{equation}
\begin{equation}
{\bf \bar x}_0 = {\bf \hat x}^ +  . 
\end{equation}
\end{subequations}
Hre the parameters $\underline{\bar x}_{0/2 } $ ($\underline{\bar t}_{0/2 } $) and $\overline{\bar x}_{0/2 } $ ($\overline{\bar t}_{0/2 } $) denote the  lower and upper bounds, respectively, of the original constraints on $ x_{0/2}$. The terms $\bar x_{det,1 } $/$\bar t_{det,1 } $, $\bar x_{det,3 } $/$\bar t_{det,3} $, and $\bar \delta_{det } $ represent the upper bounds of the tightened constraints on $\bar x_1$, $\bar x_3$, and $\bar\delta$, respectively.  These bounds are derived  online through a lookup table on the basis of the detected road curvature.  The terms $\underline{\bar x}_1$/$\overline{\bar x}_1$, $\underline{\bar t}_1$/$\overline{\bar t}_1$, $\underline{\bar x}_3$/$\overline{\bar x}_3$, $\underline{\bar t}_3$/$\overline{\bar t}_3$, and $\underline{\bar \delta}$/$\overline{\bar \delta}$ indicate the designed constraint bounds, which adhere to the following design rules  \cite{Lee26}: $0 \le \underline{\bar x}_1 < {\bar x}_{det,1} < \overline{\bar x}_1$, $0 \le \underline{\bar t}_1 \le {\bar t}_{det,1} < \overline{\bar t}_1$, $0 \le \underline{\bar x}_3 < {\bar x}_{det,3} < \overline{\bar x}_3$, $0 \le \underline{\bar t}_3 < {\bar t}_{det,3} < \overline{\bar t}_3$, and $0 \le \underline{\bar \delta} < {\bar \delta}_{det} < \overline{\bar \delta}$. Here, ${\bar x_{det}  - \underline{\bar x} } = {\overline{\bar x}  - \bar x_{det} } \equiv D \bar x_{det }$, ${\bar t_{det}  - \underline{\bar t} } = {\overline{\bar t}  - \bar t_{det} } \equiv D \bar t_{det }$, and ${\bar \delta_{det}  - \underline{\bar \delta} } = {\overline{\bar \delta}  - \bar \delta_{det} } \equiv D \bar \delta_{det }$. In addition, $\overline{\bar x}$, $\overline{\bar t}$, and $\overline{\bar \delta}$ are constrained to not exceed the actual system constraint bounds. Specifically, the interpolation weight of $\bar x_{det}$/$\bar t_{det}$/$\bar \delta_{det}$ (i.e., ${\lambda }$) is fixed at ${\lambda } =1.0-2D$, where $D$ is set to 0.22 to reduce the number of variables and thus decrease the computational cost. Finally, the scalars $q_{s}$, $q_{u}$, $q_{l1}$, and $q_{l2}$ denote penalty parameters. The penalty parameters  selected in this study are listed in Table II.

\begin{figure*}[!t]
\centerline{\includegraphics[scale=0.25]{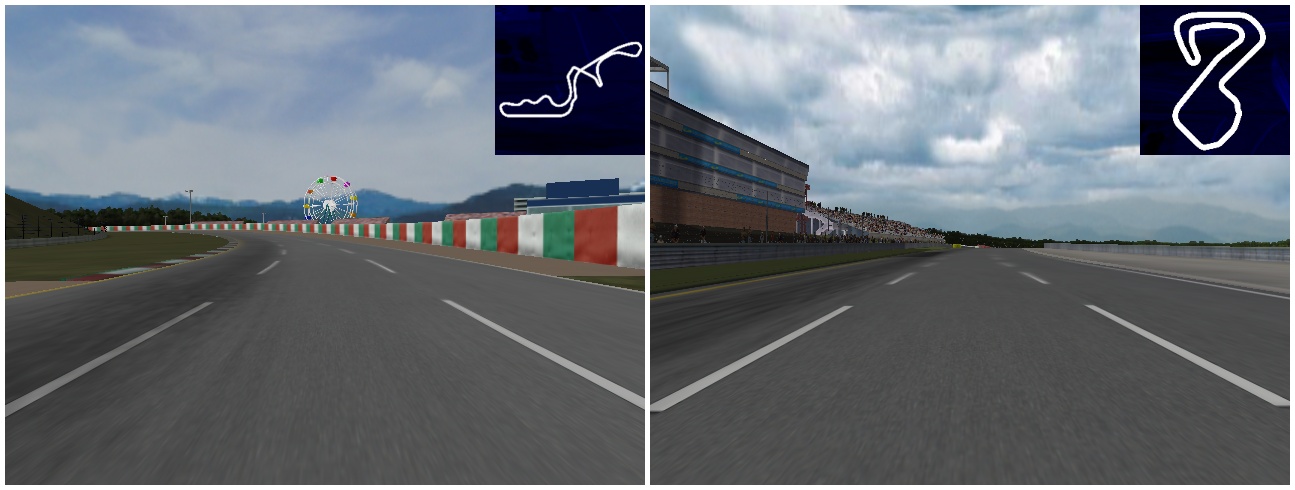}}
\caption{Example traffic scene and maps of track A (left) and track B (right) in TORCS. The input of the adopted deep neural network model was an RGB image with a resolution of 228 $\times$ 228 pixels.}
\end{figure*}

\begin{figure}[!t]
\centerline{\includegraphics[scale=0.3]{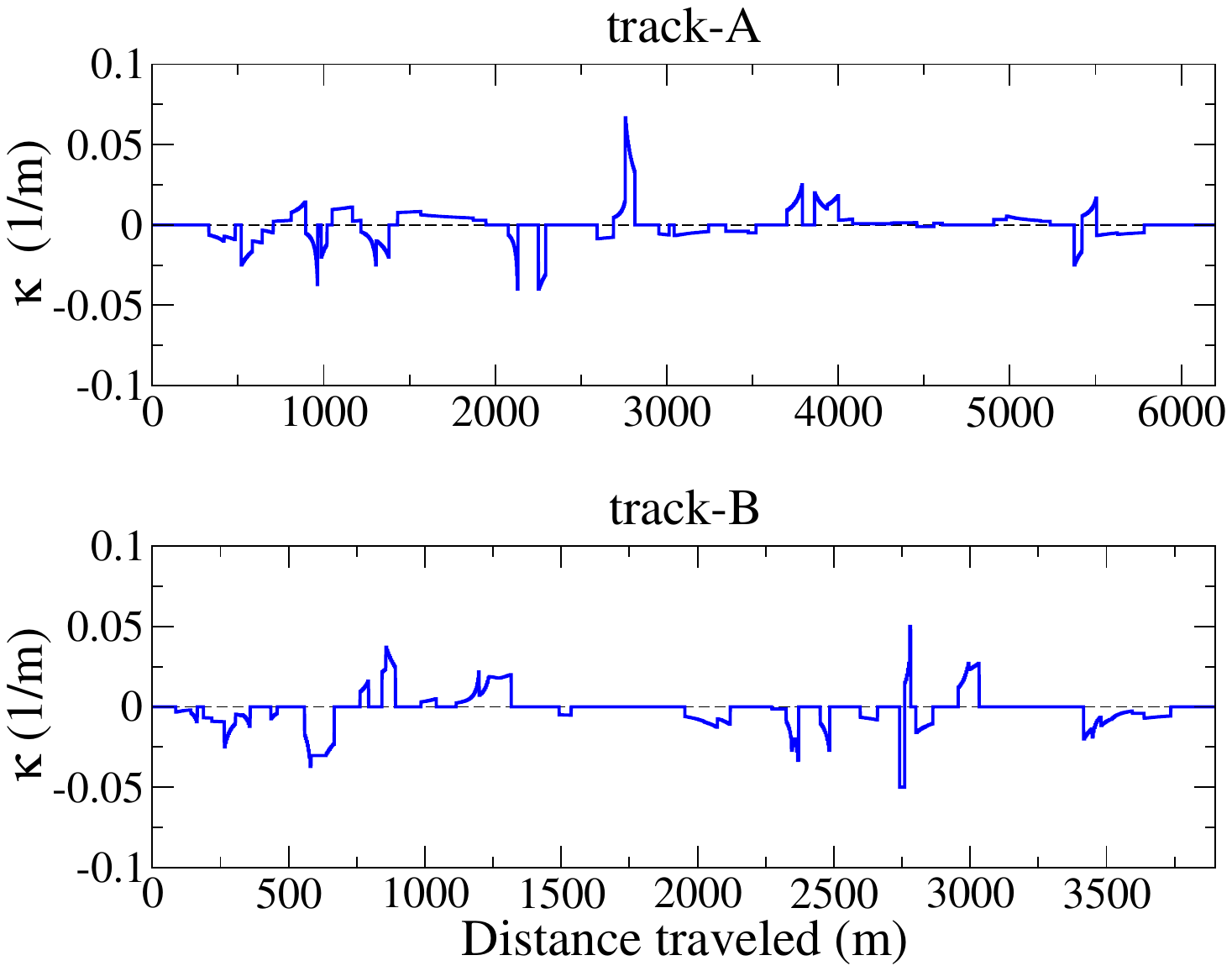}}
\caption{Curvature profiles of tracks A and B.}
\end{figure}

To solve Problem 3, the state, control, and interpolation variables are optimized iteratively through alternating backward and forward propagation steps by using the ILQR algorithm \cite{Tas14} and Newton's  method \cite{Lee25}, respectively. The corresponding update equations for these variables are expressed as follows:
\begin{subequations}
\begin{equation}
{\bf \tilde  u}_i  = {\bf \bar u}_i  + {\bf \tilde k}_i  + {\bf \tilde K}_i \left( {{\bf \tilde x}_i  - {\bf \bar x}_i } \right),
\end{equation}
\begin{equation}
{\bf \tilde x}_{i + 1}  = {\bf \bar f}\left( {{\bf \tilde x}_i ,{\bf \tilde u}_i,{\bf \hat d}^ + } \right) ,
\end{equation}
\begin{equation}
{\bf \tilde \Lambda }_i  = {\bf \Lambda }_i  - \left[ {{\bf H}\bar J_{{\mathop{ int}} } '\left( {{\bf \Lambda }_i } \right)} \right]^{ - 1} \nabla \bar J_{{\mathop{ int}} } '\left( {{\bf \Lambda }_i } \right),
\end{equation}
\end{subequations}
where the controller parameters ${\bf \tilde k}_i $ and ${\bf \tilde K}_i $ are computed using the analytical formulas presented in \cite{Tas14} and the Hessian matrix ${\bf H}\bar J_{{\mathop{ int}}}'\left( {{\bf \Lambda}_i } \right) \succ 0$. 

In addition to using the itube-CILQR algorithm to solve the regulation problem (Problem P3) in the proposed robust OF-MPC framework, this study employed two established optimization solvers, namely the Interior Point Optimizer (IPOPT) \cite{Wac06, Sha24} and the Operator Splitting Quadratic Program (OSQP) solver \cite{Ge22, Ge22c, Bar20} to solve  the regulation problem (Problem P2') for performance comparison. The experimental results are presented in the following section.

\section{Experimental Results and Discussion}
This section presents the experimental results obtained with the proposed robust OF-MPC framework. First, numerical simulations of lane-keeping control on large-curvature roads were conducted, with the control performance of the itube-CILQR algorithm compared with that of IPOPT-based and OSQP-based algorithms. Subsequently, vision-based autonomous driving experiments were conducted in TORCS \cite{Wym13, Che15, Car15, Bon17, Wu23}, which  is an open-source driving simulator  providing a high-fidelity environment suitable for training and validating autonomous driving models. The actuation delay in TORCS\ was approximately 6.66 ms on the adopted computer. This lag may have arisen from computationally intensive graphics rendering, physics-based simulations, and data serialization. In addition to road curvature, road friction and measurement noise were considered in the vision-based experiments. All experiments were
performed on a personal computer with an Intel i9-9900K CPU equipped with an Nvidia RTX 2080 Ti GPU.

\subsection{Numerical Simulations}
The lane-keeping maneuver in the numerical simulations involved an ego vehicle traveling on a road with two tight curves having curvatures ($\kappa $) of 0.08 and \Pisymbol{@symbol}{45}0.05 1/m, which were encountered at time steps 450 $\le$ $t$ $\le$ 700 and 950 $\le$ $t$ $\le$ 1200, respectively. The initial state of the ego vehicle was assumed to be ${\bf x}_0  = {\bf \hat x}_0  =\left[ {\begin{array}{*{20}c}
   2.0 & 0.0 & 0.0 & 0.0  \\
\end{array}} \right]^T$, with the initial disturbance vector set to ${\bf \hat d}_0  =\left[ {\begin{array}{*{20}c}
   {0.0} & {0.0}  \\
\end{array}} \right]^T $. Moreover, the ego vehicle had a constant speed of 20 m/s.

Fig. 4 presents the numerical simulation results obtained using IPOPT- and OSQP-based solvers. These solvers could not achieve convergence when the optimization problem incorporated terminal constraints with a scaling of $c_i $ = $10^{-16}$ or 20 because   these constraints were overly restrictive and prevented both solvers from finding feasible solutions (cf. Fig. 3). To address this limitation, we replaced  the aforementioned terminal constraints with  less restrictive stage constraints in the IPOPT- and OSQP-based solvers to develop the tube-IPOPT-X and tube-OSQP-X solvers, respectively. However, system stability was not guaranteed with these stage constraints, which approximated the original terminal constraints. The performance of the tube-IPOPT-X and tube-OSQP-X solvers was compared with that of a nominal-IPOPT solver for which the disturbance vector and steady-state targets were set to zeros in the OF-MPC framework.

Fig. 4(a) and 4(b) displays the trajectories of the system states ($\Delta$) and control inputs ($\delta$) obtained with the three solvers, respectively. As indicated in Fig. 4(a), the nominal-IPOPT solver generated  a plateau-like trajectory at both turns because of the corresponding road curvatures, with the maximum lateral offset ($\Delta$) being approximately 0.5 m. By contrast, the tube-IPOPT-X and tube-OSQP-X solvers produced nearly zero lateral offset at both turns. However, they generated small transient overshoots when the ego vehicle entered or exited a turn. As presented in Table III, the tube-IPOPT-X solver required an average computation time of 6.45 ms to produce a solution with a maximum lateral offset of 0.024333 m around $t$ = 720. By comparison, the  tube-OSQP-X solver produced a smaller maximum lateral offset  of 0.024268 m around $t$ = 720 within an average computation time of 1 ms. Because the uncertainty parameter structure of ${\bf \hat d }$ was known, the road curvatures were estimated by dividing  ${\hat d_0 }$ and ${\hat d_1 }$ by  $\gamma _0$ and $\gamma _1$, respectively. Fig. 4(c) illustrates the trajectories of the estimated road curvatures obtained with the tube-OSQP-X algorithm. The OF-MPC framework enabled both solvers to  accurately estimate the overall curvature profiles, with overshoot limited to only transient intervals [Fig. 4(a) and 4(b)]. Visual information was only used to determine tightened constraints, whereas road curvature estimation exclusively relied on the optimization of the disturbance vector in the OF-MPC framework. Accordingly, the OF-MPC framework can accurately estimate road curvature under minimum internal or external disturbance. Consequently, we used this framework to determine the perceptual uncertainty of the adopted visual perception system.

\begin{figure}[!t]
\centerline{\includegraphics[scale=0.3]{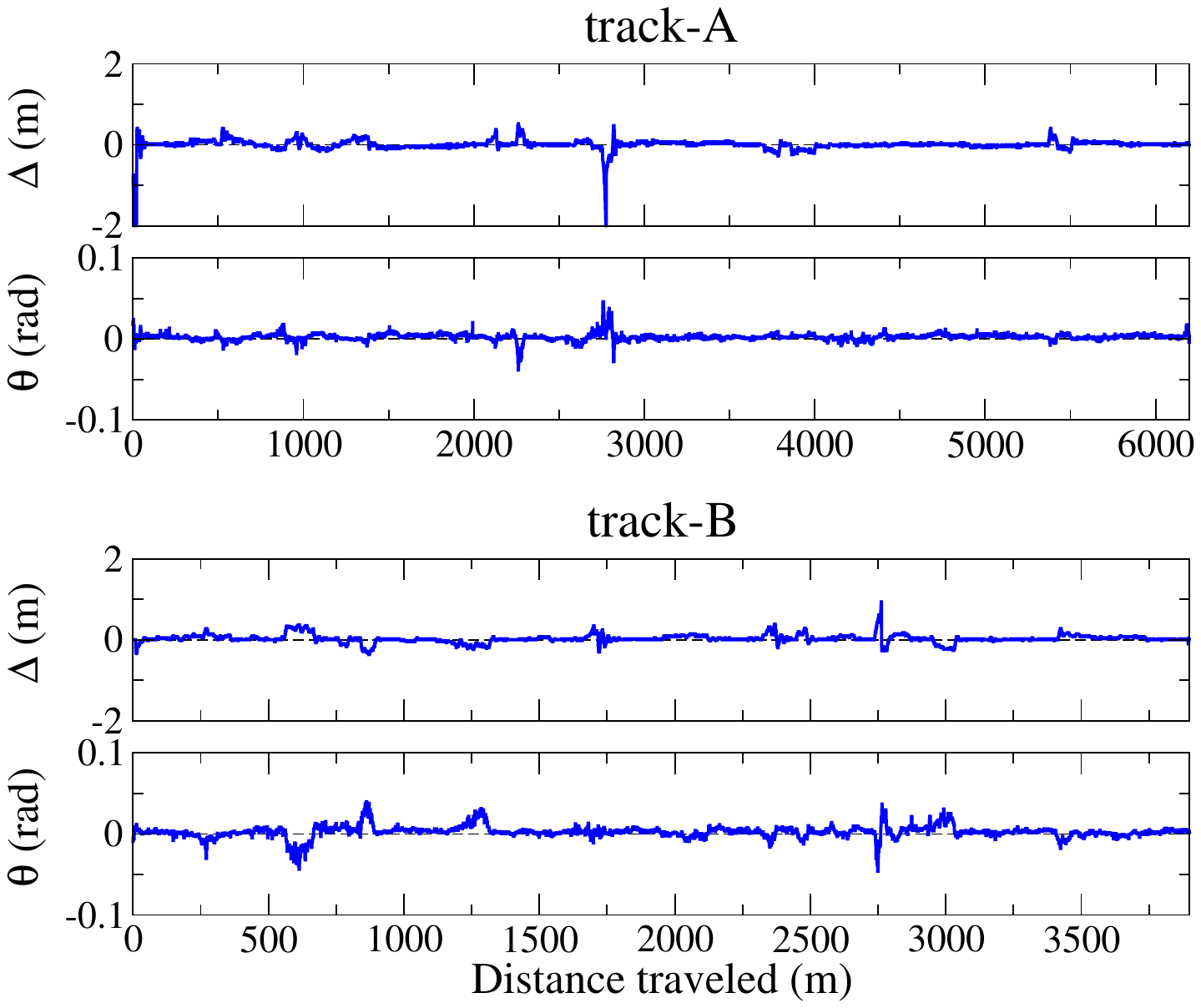}}
\caption{Dynamic performance of the itube-CILQR-T-20 algorithm in lane-keeping maneuvers within the central lanes of tracks A and B at a vehicle speed  of 72 km/h when the noise level ($\sigma$) was set to 0 and the friction coefficient ($fr$) was set to 1.0.}
\end{figure}

Figs. 5 and 6 present the trajectories of all state, control, and disturbance variables obtained using three  itube-CILQR-based solvers. Nearly offset-free trajectory tracking was achieved for all state variables ($x_0$, $x_1$, $x_2$, $x_3$), except during transient intervals, consistent with the results displayed in Fig. 4. Similar to the tube-IPOPT-X and tube-OSQP-X solvers, the itube-CILQR-X solver employed stage constraints instead of terminal constraints. By contrast, the itube-CILQR-T-20 and itube-CILQR-T-0 solvers used  ellipsoid terminal sets with  scaling constraints of $c_i$ = 20.0 and $10^{-16}$, respectively. These itube-CILQR solvers were developed by optimizing the design introduced in \cite{Lee26} to achieve  considerably higher computation speeds. As indicated in Table III, the itube-CILQR-X solver exhibited an average computation time slightly exceeding 1 ms, which is shorter than that reported in \cite{Lee26}. Moreover, by selecting suitable barrier parameters ($q_{s}$, $q_{u}$, $q_{l1}$, and $q_{l2}$), this solver reduced the maximum lateral offset to  0.023162 m, thereby outperforming the tube-IPOPT-X  and tube-OSQP-X solvers. The itube-CILQR-T-20 and itube-CILQR-T-0 solvers achieved convergence with terminal ellipsoid constraints, which not only guaranteed the stability of the controlled
system but also enhanced lateral control performance. The maximum lateral offset  decreased to 0.022097 and 0.021675 m when the itube-CILQR-T-20 and itube-CILQR-T-0 solvers were employed, respectively.\ The itube-CILQR-T-0 solver with $c_i$ = $10^{-16}$ exhibited a smaller terminal constraint region and produced solutions with lower uncertainty than did the itube-CILQR-T-20 solver with $c_i$ = 20. However, a trade-off between computational burden and performance was observed, with the average computation time of the itube-CILQR-T-0 solver being 1.74 times longer than that of the itube-CILQR-T-20 solver. Because the actuation delay in TORCS on the adopted computer was approximately 6.66 ms, we used the itube-CILQR-T-20 solver to adequately balance between control accuracy and computational efficiency in the vision-based autonomous driving experiments.

\begin{figure}[!t]
\centerline{\includegraphics[scale=0.3]{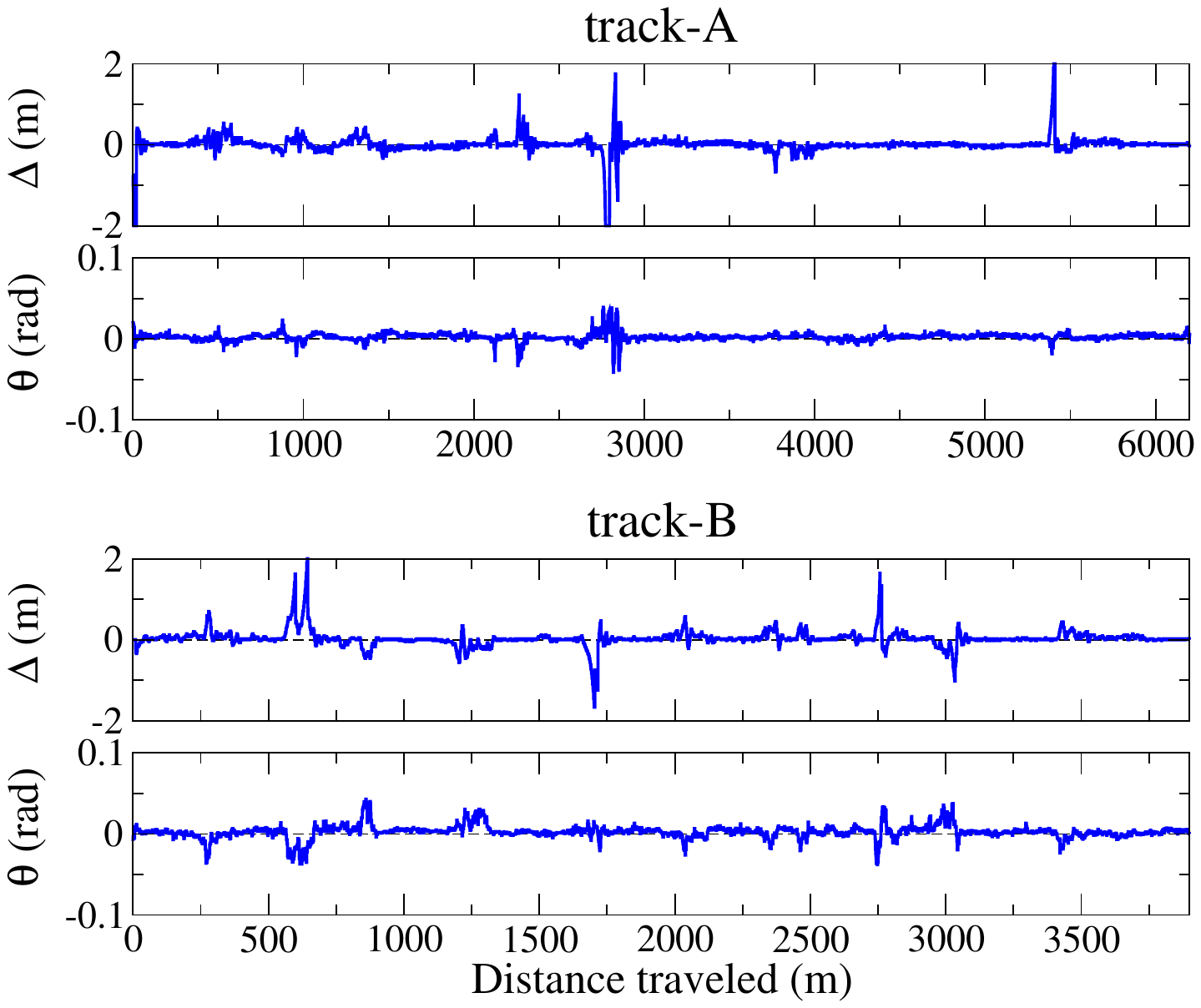}}
\caption{Dynamic performance of the itube-CILQR-T-20 algorithm in lane-keeping maneuvers within the central lanes of tracks A and B at a vehicle speed  of 72 km/h when $\sigma$ = 1.8$\%$ and 2.0$\%$ and $fr$ = 1.0.}
\end{figure}

\begin{table*}[!t]
\caption{Average and Standard Deviation Values of  the $\kappa$-Score During Lane-Keeping Maneuvers Conducted Using the itube-CILQR-T-20 Solver at Different $\sigma$ and $fr$ Values}
\begin{center}
\begin{tabular}{|c|c|c|c|c|c|c|c|c|c|c|}\hline
\multirow{3}{*}{$\sigma$}& \multicolumn{5}{c|}{track-A}   & \multicolumn{5}{c|}{track-B}   \\ \cline{2-11} 
 & \multicolumn{3}{c|}{$fr$} &\multirow{2}{*}{Avg}&\multirow{2}{*}{SD}  &\multicolumn{3}{c|}{$fr$}&\multirow{2}{*}{Avg}&\multirow{2}{*}{SD} \\ \cline{2-4} \cline{7-9}
  & 1.0 & 0.9 & 0.8 &   &  & 1.0 & 0.9 & 0.8&   &   \\ \hline
0.5 $\%$ &0.9941 &0.9159 & 0.9324&0.9474 &0.0412 &0.8604 &0.8616 &0.8720 &0.8646 &0.0063 \\ \hline

0.8 $\%$&1.4989 & 1.4021&1.4266 &1.4425 &0.0503 &1.3625 &1.3294 &1.2987 &1.3302 &0.0319 \\ \hline

1.0 $\% $&1.7393 &1.6674 & 1.6334&1.6800 &0.0540 &1.5741 &1.6156 &1.4961 &1.5619 &0.0606 \\ \hline

1.3 $\%$& 1.9226&1.9093 & 1.9523&1.9280 & 0.0220&1.8943 &1.8633 &1.6802 &1.8126 &0.1157 \\ \hline

1.5 $\%$& 2.0771&1.9538 &1.8004 &1.9437 & 0.1386& 2.0595 &1.7904 &1.4862 &1.7787 &0.2868 \\ \hline

1.8 $\%$&  1.8727& 1.7448& Fail&1.8087 &0.0904 &1.8432 &1.6311 &1.2930 &1.5891 &0.2774 \\ \hline

2.0 $\%$&Fail & Fail&-- &-- &-- &1.8593 &Fail &Fail &-- &-- \\ \hline

2.2 $\%$& --&-- &-- &-- &-- &1.3220 &-- &-- &-- &-- \\ \hline

2.3 $\%$&-- & --&-- & --&-- &Fail &-- &-- &-- &-- \\ \hline
\end{tabular}
\end{center}
\end{table*}

\subsection{Vision-Guided Lane-Keeping Experiments}
This study evaluated the performance of the itube-CILQR-T-20 solver in autonomous vehicle lane-keeping experiments within a TORCS environment. The adopted visual perception system incorporated a deep neural network model, namely multitask UNet (MTUNet) \cite{Lee24, Ron15}. To extract visual features, this model processes images  captured by a front-facing dashboard camera mounted on the windshield of an ego vehicle at a rate of 40 frames per second. On the basis of the extracted features, the MTUNet model estimates the ego vehicle's heading angle $\theta$ and generates a binary lane map in the egocentric view. Postprocessing techniques are then used to determine the immediate lateral offset $\Delta$ and road curvature $\kappa$ relative to the current ego vehicle position. The estimated $\Delta$ and $\theta$ values constitute the output vector ${\bf y} \equiv  \left[ {\begin{array}{*{20}c}
   \Delta  & \theta   \\
\end{array}} \right]^T$, which was used by the OF-MPC framework in this study to achieve accurate road centerline tracking. The itube-CILQR-T-20 solver operated under tightened constraints, which were determined from an offline precomputed look-up table on the basis of the $\kappa$ value estimated online \cite{Lee26}. Additionally, the longitudinal velocity of the ego vehicle was independently regulated by a proportional\textendash integral controller \cite{Lee24} to maintain a constant speed of 72 km/h (20 m/s). To rigorously assess the performance and robustness of the itube-CILQR-T-20 solver under varying environmental conditions, experiments were conducted on 2 three-lane highway tracks in TORCS, namely tracks A and B, which had a lane width of 4 m. The  total lengths of tracks A and B were 6205 and 3919 m, respectively, and their corresponding maximum curvatures were approximately 0.06 and 0.05 1/m. Figs. 7 and 8 depict an example traffic scene, the maps of tracks A\ and B, and the curvature profiles of these tracks. The road friction coefficient ($fr$) was varied between 0.8 and 1.0 to reflect different road adhesion conditions in TORCS \cite{Bon17}.

White noise was introduced into $\Delta$ and $\theta$ at each time step to generate the noisy measurement vector ${\bf \hat y}$, which is expressed as follows:  
\begin{equation}
{\bf \hat y}_i  = \left[ {\begin{array}{*{20}c}
   {\left( {1 + \sigma _i } \right)\Delta } & {\left( {1 + \sigma '_i } \right)\theta }  \\
\end{array}} \right]^T ,
\end{equation}
where $\sigma _i$ and $\sigma '_i$ are  uniformly distributed random numbers in the range of [\Pisymbol{@symbol}{45}$ \sigma,  \sigma$], where $\sigma$ is the noise level. The same noise level was assumed for both $\Delta$ and $\theta$. Subsequently, ${\bf y}$, $\bf \hat y,$ and the disturbance vector  ${\bf \hat d}$  (a zero vector) were used as inputs to the Kalman filter to estimate the  updated state and disturbance vectors (${\bf \hat x}^ +  $ and ${\bf \hat d}^ +  $, respectively). Because the OF-MPC framework accurately estimated road curvature under minimum disturbance in the numerical simulations (Sec. V-A), we  used this framework to determine the perceptual uncertainty of  the adopted visual perception framework \cite{Lee24} based on this observation. For this purpose, we defined a $\kappa$-score as the ratio of the noise level ($\sigma$) to the mean absolute difference between the true curvature $\kappa$ and the average estimated curvature [0.5 \Pisymbol{@symbol}{180} ($\hat d^+_{0} /\gamma _0$ + $\hat d^+_{1} /\gamma _1$)] across all time steps in a complete lane-keeping maneuver. The $\kappa$-score is expressed as follows:
\begin{equation}
\kappa {\rm  - score} \equiv \left( {\frac{1}{\sigma }\sum\limits_{i = 1}^M {\frac{{\left| {\kappa _i  - 0.5\left( {\hat d^+_{0,i} /\gamma _0  + \hat d^+_{1,i} /\gamma _1 } \right)} \right|}}{M}} } \right)^{ - 1},
\end{equation}
where $M$ is the total number of data points in a complete lane-keeping maneuver. Because $\hat d^+_{0} $ and $\hat d^+_{1} $ were estimated using the Kalman filter, the $\kappa$-score inherently reflected information obtained through the fusion of $\Delta$ and $\theta;$   thus, this metric may be more accurate and reliable than those derived from either indicator alone. We varied the noise level systematically to identify the $\sigma$ value yielding the  highest  $\kappa$-score, with this value serving as an estimate of the uncertainty of the adopted visual perception framework.  This estimation was based on the assumption that the uncertainty resulted from white noise.

 When $fr$ = 1.0, the itube-CILQR-T-20 solver, which incorporated the OF-MPC framework, effectively maintained the ego vehicle within the center lane  over one complete lap on tracks A and B under noise-free conditions ($\sigma$ = 0; Fig. 9). It also ensured accurate lane-keeping control on tracks A and B at noise levels ($\sigma$) of up to 1.8$\%$ and 2.0$\%$  (Fig. 10), respectively, when $fr$ = 1.0. However, as presented in Table IV, the ego vehicle deviated from the central lane at the curved sections of both tracks when $\sigma$ exceeded the aforementioned values. When $fr$ = 1.0, 0.9, and 0.8, the ego vehicle failed to complete the lane-keeping maneuver at $\sigma$ values equaling or exceeding 2.0$\%$, 2.0$\%$, and 1.8$\%$, respectively, on track A and 2.3$\%$, 2.0$\%$, and 2.0$\%$, respectively, on track B.  These results indicate that lane-keeping failed at lower noise levels on track A  than on track B, which may be attributed to the lower image quality for track A (Fig. 7). The $\kappa$-scores for each maneuver were averaged across all $fr$ values, with the highest mean scores on tracks A and B occurring when $\sigma$ = 1.5$\%$ and 1.3$\%$, respectively. These noise levels represent the perceptual uncertainty of the adopted vision perception system on tracks A and B, which were not used in the network training stage.

\section{Conclusion}
This study developed a robust OF-MPC framework for ensuring lane-keeping control during high-speed cornering maneuvers in autonomous driving scenarios in which only a vision sensor is used to perceive driving scenes   \cite{Raw03, Lee26, Lee24, Raw17, Lee25}. The regulation problem in this control framework can be solved using various solvers, including CILQR-based, IPOPT-based, and OSQP-based solvers. Numerical simulations revealed that the itube-CILQR-T-20 solver exhibited the best balance between control performance and computational efficiency among all compared solvers. It ensured robust, stable, and offset-free tracking of the controlled system and demonstrated an average computation time of 4.21 ms.  Therefore, this solver was used in vision-based experiments conducted in TORCS environments to quantify the uncertainty of a visual sensing model (MTUNet \cite{Lee24}) on highway tracks not encountered during model training. Furthermore, a $\kappa$-score was defined to measure the perceptual uncertainty of this model at various road friction and noise levels in the proposed OF-MPC\ framework. The uncertainty (noise level) of the model was 1.5$\%$ on track A and 1.3$\%$ on track B. Overall, the developed robust OF-MPC\ framework and proposed uncertainty estimation method can provide a foundation for the future development of advanced vision-based autonomous driving systems.

\appendix
The appendix provides proofs of the closed-loop properties of the regulation problem in the proposed OF-MPC framework. Specifically, the recursive feasibility and input-to-state stability (ISS) are analyzed under the stated assumptions. The notation used here is the same as that adopted for Problem P2'.

\section*{I. Recursive Feasibility}
\setcounter{equation}{0}
\renewcommand{\theequation}{A-\arabic{equation}}

\begin{theorem}
Assume that steady-state targets exist and that the standard axioms of the classical MPC problem are satisfied. Suppose that the MPC problem is feasible at the current time step $t$. For a fixed terminal ellipsoid with a scaling parameter $c_f$, the optimization problem remains feasible at all subsequent time steps $t+1$, despite arbitrary bounded shifts in the steady-state targets.
\end{theorem}
\textit{Proof:} Suppose that the optimal control sequence for solving Problem P2' under the estimated disturbance ${\bf \hat d}$ has been determined at time step $t$  and is given by ${\bf \bar U}_t^*  = \left\{ {{\bf \bar u}_{\left. 0 \right|t}^* ,{\bf \bar u}_{\left. 1 \right|t}^* ,...,{\bf \bar u}_{\left. {N - 1} \right|t}^* } \right\}$. Let ${\bf \bar X}_t^*$ denote the optimal state trajectory, which is expressed as ${\bf \bar X}_t^*  = \left\{ {{\bf \bar x}_{0} ,{\bf \bar x}_{\left. 1 \right|t}^* ,...,{\bf \bar x}_{\left. N \right|t}^* } \right\}$, where ${\bf \bar x}_{0}$ represents the initial state, ${\bf \bar z}_{s,t} \equiv \left[ {{\bf \bar x}^T_{s,t} ,{\bf \bar u}^T_{s,t} } \right]^T$ represents the steady-state target,  and ${\bf \bar x}_{\left. N \right|t}  \in \mathbb{\bar T } $ represents the terminal state. The first element of ${\bf \bar U}_t^*$, namely ${\bf \bar u}_{\left. 0 \right|t}$, drives the initial state ${\bf \bar x}_{0}$ to the successor state ${\bf \bar x}_{\left. 1 \right|t}$.  At time step $t + 1$, the steady-state target is ${\bf \bar z}_{s,t+1} $, and a candidate feasible control sequence is generated from the previous optimal solution. This sequence is written as ${\bf \bar U}_{t + 1}  = \left\{ {{\bf \bar u}_{\left. 1 \right|t}^* ,...,{\bf \bar u}_{\left. {N - 1} \right|t}^* ,{\bf \kappa }_f } \right\}$, where ${\bf \kappa }_f$ is the terminal control policy. The resulting state trajectory is expressed as ${\bf \bar X}_{t + 1}  = \left\{ {{\bf \bar x}_{\left. 1 \right|t}^* ,...,{\bf \bar x}_{\left. N \right|t}^* ,{\bf \bar f}\left( {{\bf \bar x}^*_{\left. N \right|t} ,{\bf \kappa }_f,{\bf \hat d}^+ } \right)  } \right\}$, which  fulfills the initial nominal state constraint and the corresponding stage constraints. Furthermore, because  the maximum ellipsoidal set $\mathbb{\bar T }$ is positively invariant,  ${\bf \bar x}_{\left. N \right|t}^* \in \mathbb{\bar T }$, and ${\bf \hat d}^+  \in  \mathbb{W}$ $\Rightarrow$ ${\bf \bar f}\left( {{\bf \bar x}^*_{\left. N \right|t} ,{\bf \kappa }_f,{\bf \hat d}^+ } \right) \in \mathbb{\bar T }$, the terminal constraint of ${\bf \bar X}_{t + 1}$ is satisfied. Therefore, at least  one solution exists at time $t+1$, which ensures the recursive feasibility of the regulation problem. $\blacksquare$

\section*{II. Input-to-State Stability }
\setcounter{equation}{0}
\renewcommand{\theequation}{B-\arabic{equation}}

\begin{theorem}
Under the terminal ellipsoid constraint $\mathbb{\bar T }$, the closed-loop system exhibits ISS with respect to the steady-state target variation rate ${\bf \delta \bar x}_s$.
\end{theorem}
\textit{Proof:} Let $J^{\rm *} \left( {{\bf \bar X}_{t}^* ,{\bf \bar U}_{t}^* ,{\bf \bar z}_{s,t} } \right)$ denote the optimal cost function at time step $t$, which serves as a candidate Lyapunov function. At time step $t+1$, the  optimal cost function must be less than or equal to the cost evaluated using a candidate feasible sequence as follows:
\begin{equation}
J^{\rm *} \left( {{\bf \bar X}_{t + 1}^* ,{\bf \bar U}_{t + 1}^* ,{\bf \bar z}_{s,t + 1} } \right) \le J\left( {{\bf \bar X}_{t + 1} ,{\bf \bar U}_{t + 1} ,{\bf \bar z}_{s,t + 1} } \right).
\end{equation}
To determine whether the optimal cost function at time $t+1$ is better (lower) than that that at time $t$, the difference between these functions is calculated as follows:
\begin{equation}
\begin{split}
&J^{\rm *} \left( {{\bf \bar X}_{t + 1}^* ,{\bf \bar U}_{t + 1}^* ,{\bf \bar z}_{s,t + 1} } \right) - J^{\rm *} \left( {{\bf \bar X}_t^* ,{\bf \bar U}_t^* ,{\bf \bar z}_{s,t} } \right) \\& = \Delta J_x - J_x^{\rm *} \left( {{\bf \bar x}_0  , {\bf \bar x}_{s,t} } \right) - J_f^* \left( {{\bf \bar x}_{\left. N \right|t}^*  , {\bf \bar x}_{s,t} } \right) \\&+ J_x^* \left( {{\bf \bar x}_{\left. {N - 1} \right|t + 1}^*  , {\bf \bar x}_{s,t + 1} } \right) + J_f^* \left( {{\bf \bar x}_{\left. N \right|t + 1}^*  , {\bf \bar x}_{s,t + 1} } \right).
\end{split}
\end{equation}
Here, $\Delta J_x$ refers to the change in the state stage cost caused by target switching, which is bounded by the maximum target variation. The variation in the control stage cost ($\Delta J_u$) is ignored here for simplicity, and $\Delta J_u$ can be analyzed using the same approach as that used for $\Delta J_x$. The last two terms of (B-2) fulfill  Lyapunov descent inequality, as expressed in the following equation:
\begin{equation}
\begin{split}
&J_x^* \left( {{\bf \bar x}_{\left. {N - 1} \right|t + 1}^*  , {\bf \bar x}_{s,t + 1} } \right) + J_f^* \left( {{\bf \bar x}_{\left. N \right|t + 1}^*  , {\bf \bar x}_{s,t + 1} } \right) \\
&\le J_f^* \left( {{\bf \bar x}_{\left. N \right|t}^*  , {\bf \bar x}_{s,t + 1} } \right).
\end{split}
\end{equation}
Combining (B-2) and (B-3) yields the following:
\begin{equation}
\begin{split}
&J^{\rm *} \left( {{\bf \bar X}_{t + 1}^* ,{\bf \bar U}_{t + 1}^* ,{\bf \bar z}_{s,t + 1} } \right) -  J^{\rm *} \left( {{\bf \bar X}_t^* ,{\bf \bar U}_t^* ,{\bf \bar z}_{s,t} } \right) \\&\le \Delta J_x - J_x^{\rm *} \left( {{\bf \bar x}_0  , {\bf \bar x}_{s,t} } \right) - J_f^* \left( {{\bf \bar x}_{\left. N \right|t}^*  , {\bf \bar x}_{s,t} } \right) \\&+ J_f^* \left( {{\bf \bar x}_{\left. N \right|t}^*  , {\bf \bar x}_{s,t + 1} } \right) 
\end{split}
\end{equation}
Under the Lipschitz continuity of the terminal cost and the bound $\left\| {{\bf \delta \bar x}_s } \right\|_{\max } $, the terms of (B-4) are bounded as follows:
\begin{equation}
\begin{split}
&\Delta J_x - J_f^* \left( {{\bf \bar x}_{\left. N \right|t}^* ,{\bf \bar x}_{s,t} } \right)+ J_f^* \left( {{\bf \bar x}_{\left. N \right|t}^* ,{\bf \bar x}_{s,t + 1} } \right)  \\&\le \Delta J_x +\alpha\left\| {{\bf \delta \bar x}_s } \right\|_{\max } \le  \beta \left\| {{\bf \delta \bar x}_s } \right\|_{\max }, 
\end{split}
\end{equation}
where $\alpha > 0$ and $\beta >0$ are constants. Combining the stage cost drop and the maximum target variation bound yields the ISS Lyapunov inequality, which is expressed as follows:
\begin{equation}
\begin{split}
&J^{\rm *} \left( {{\bf \bar X}_{t + 1}^* ,{\bf \bar U}_{t + 1}^* ,{\bf \bar z}_{s,t + 1} } \right)  - J^{\rm *} \left( {{\bf \bar X}_t^* ,{\bf \bar U}_t^* ,{\bf \bar z}_{s,t} } \right) \\& \le - \left\| {{\bf \bar x}_0  - {\bf \bar x}_{s,t} } \right\|_{\bf Q}^2  +  \beta \left\| {{\bf \delta \bar x}_s } \right\|_{\max }. 
\end{split}
\end{equation}
Thus, the cost function remains bounded around the origin, resulting in ISS with respect to disturbance variations. $\blacksquare$



\end{document}